\documentclass{article} % For LaTeX2e
\usepackage{iclr2026_conference,times}
\usepackage[utf8]{inputenc} % allow utf-8 input
\usepackage[T1]{fontenc}    % use 8-bit T1 fonts
\usepackage{hyperref}       % hyperlinks
\usepackage{url}            % simple URL typesetting
\usepackage{booktabs}       % professional-quality tables
\usepackage{amsfonts}       % blackboard math symbols
\usepackage{nicefrac}       % compact symbols for 1/2, etc.
\usepackage{microtype}      % microtypography
\usepackage[svgnames,x11names,table]{xcolor}         % colors

\usepackage[textsize=tiny]{todonotes}

\usepackage{amsmath, amsthm, amsopn, amstext, amsfonts, bm, algorithm, algpseudocode, graphicx, mathabx}

\newtheorem{theorem}{Theorem}[section]

\newtheorem{lemma}[theorem]{Lemma}
\newtheorem{corollary}[theorem]{Corollary}
\newtheorem{definition}[theorem]{Definition}

\newtheorem{assumption}[theorem]{Assumption}

\def\eqref#1{equation~\ref{#1}}
\def\Eqref#1{Equation~\ref{#1}}
\def\1{\bm{1}}

\DeclareMathAlphabet{\mathsfit}{\encodingdefault}{\sfdefault}{m}{sl}
\SetMathAlphabet{\mathsfit}{bold}{\encodingdefault}{\sfdefault}{bx}{n}

\newcommand{\E}{\mathbb{E}}

\newcommand{\R}{\mathbb{R}}

\DeclareMathOperator*{\argmin}{arg\,min}

\newcommand{\xb}{\bar{x}}

\newcommand{\quant}{\mathcal{Q}}

\usepackage{subcaption}     % subcaption for multiple images
\usepackage{multirow}
\usepackage{colortbl}
\usepackage{bbm}
\usepackage{tikz}
\usepackage{pifont}
\newcommand{\xmark}{\ding{55}} % try 53, 54, 55, 56
\usepackage{hhline}
\usepackage{framed}
\usepackage{float}
\usepackage{enumitem}
\usepackage[capitalize,noabbrev]{cleveref}
\usepackage{float}

\usepackage{mathtools}
\usepackage{dcolumn}
\newcolumntype{d}{D{.}{.}{-1}} % align on decimal point (auto width)

\Crefname{section}{Section}{Sections}
\Crefname{table}{Table}{Tables}
\Crefname{figure}{Figure}{Figures}

\hypersetup{
    colorlinks = true,
    linkcolor = red,
    anchorcolor = red,
    urlcolor = red,
    citecolor = MediumBlue,
}

\usepackage{xcolor}
\usepackage{xspace}
\makeatletter
\DeclareRobustCommand\onedot{\futurelet\@let@token\@onedot}
\def\@onedot{\ifx\@let@token.\else.\null\fi\xspace}
\def\eg{\emph{e.g}\onedot} 
\def\ie{\emph{i.e}\onedot} \def\Ie{\emph{I.e}\onedot}

\makeatother
\usepackage{wrapfig}

\def\elsa{\textsc{Elsa}\xspace}
\def\ours{\elsa{}\xspace}
\def\elsaq{\textsc{Elsa}$_{\text{-L}}$\xspace}
\def\oursq{\elsaq{}\xspace}

\def\eg{\emph{e.g}\onedot} 
\def\ie{\emph{i.e}\onedot} \def\Ie{\emph{I.e}\onedot}

\usepackage{xparse}
\NewDocumentCommand{\alignednum}{m o}{%
  \makebox[5em][c]{%
    \makebox[0pt][r]{#1}%
    \makebox[0pt][l]{%
        \IfValueTF{#2}
          {$_{\pm#2}$}
          {\phantom{$_{\pm0.00}$}}
    }
  }%
}

\title{The Unseen Frontier: Pushing the Limits of\\LLM Sparsity with Surrogate-Free ADMM}
\author{Kwanhee Lee\textsuperscript{1},
Hyeondo Jang\textsuperscript{1},
Dongyeop Lee\textsuperscript{1},
Dan Alistarh\textsuperscript{2},
Namhoon Lee\textsuperscript{1} \\
\textsuperscript{1}POSTECH \quad \textsuperscript{2}ISTA \\
\texttt{\{kwanhee.lee,hyeondo.jang,dongyeop.lee2,namhoon.lee\}@postech.ac.kr},\\
\texttt{dan.alistarh@ist.ac.at}
}

\iclrfinalcopy % Uncomment for camera-ready version, but NOT for submission.
\begin{document}

\maketitle

\vspace{-1em}
\begin{abstract}
% \elsa, do you want to prune LLMs? % (frozen parody)
% We 
Neural network pruning is a promising technique to mitigate the excessive computational and memory requirements of large language models (LLMs).
Despite its promise, however, progress in this area has diminished, as conventional methods are seemingly unable to surpass moderate sparsity levels (50-60\%) without severely degrading model accuracy.
This work breaks through the current impasse, presenting a principled and effective method called \ours, which achieves extreme sparsity levels of up to 90\% while retaining high model fidelity.
This is done by identifying several limitations in current practice, all of which can be traced back to their reliance on a surrogate objective formulation.
\ours tackles this issue directly and effectively via standard and well-established constrained optimization techniques based on ADMM.
Our extensive experiments across a wide range of models and scales show that \ours achieves substantial improvements over existing methods;
\eg, it achieves 7.8$\times$ less perplexity than the best existing method on LLaMA-2-7B at 90\% sparsity.
Moreover, we show that \ours remains stable even at extreme sparsity (e.g., 95\%), yielding up to $\times$3.98 inference speedup and $\times$7.80 memory compression over its dense counterpart.
We also present
\oursq, a quantized variant that scales to extremely large models (27B), and establish its theoretical convergence guarantees.
These results highlight meaningful progress in advancing the frontier of LLM sparsity, while promising that significant opportunities for further advancement may remain in directions that have so far attracted limited exploration.

\end{abstract}
\section{Introduction}

Large language models (LLMs) have become indispensable tools across various fields, from creative industries to scientific research, but their immense size incurs a tremendous amount of memory, computation, and energy consumption, posing a significant challenge to their widespread deployment \citep{kaplan2020scaling,bommasani2021opportunities,faiz2024llmcarbon}.
Neural network pruning can offer a viable solution to this problem by removing redundant parameters without compromising performance \citep{lecun1989optimal,han2015learning,hoefler2021sparsity}.
Indeed, the research community has responded to this challenge with a surge of innovative methodologies, demonstrating that LLMs can be made more compact and efficient through effective pruning techniques \citep{frantar2023sparsegpt, sunsimple, boza2024fast,meng2024alps,fang2024maskllm,liu2025proxsparse,lee2025safe}.

However, the community is witnessing a major roadblock: current methodologies are failing to push beyond a moderate level of sparsity (roughly 50-60\%) without a significant decline in model performance;
for instance, prior works have highlighted this limitation with rather incremental improvements at high sparsity \citep{meng2024alps,boza2024fast,yin2024outlier,huang2025determining}.

\begin{center}
\emph{Have we truly reached a plateau, or is there a path to continued progress?}    
\end{center}

This work provides a positive answer.
We demonstrate that it is possible to prune LLMs for very high sparsity levels---up to almost 90\%---without significant performance degradation (see \cref{fig:teaser}).

The key to our success is identifying and addressing potentially critical flaws in the current practice.
Specifically, the majority of existing methods relies on the principle of sequential layerwise reconstruction error minimization, an approach proven effective in memory-constrained environments.
However, this approach is inherently prone to propagating compounding errors while enforcing unnecessarily strong conditions and, in fact, seeks only local solutions by design based on a surrogate objective \citep{shin2024rethinking,bai2024sparsellm,huang2025determining}.
On the other hand, we suggest finding more globally optimal solutions directly by formulating a sparsity-constrained optimization problem and developing a robust solver as a whole.

We show that our approach can be applied to a wide range of LLM models and scales from 125M to 13B number of parameters.
% Our method significantly outperforms existing state-of-the-art techniques, achieving perplexity levels at least 5$\times$ and \textcolor{red}{maintains perplexity even at extreme sparsity}, alongside zero-shot prediction accuracy improvements of nearly 6\% on pruned models at 90\% sparsity.
Across this range, \ours remains stable in the highly sparse regime, whereas competing methods frequently collapse with order-of-magnitude perplexity blow-ups.
Importantly, these gains translate into practical benefits: we obtain up to $\times 7.80$ memory reduction and $\times 3.98$ inference speedup without degrading model usability.
We provide a flexible implementation as well, which incorporates memory-efficient designs including quantized optimizer states and enables pruning even for 27B-parameter models with 55\% lower memory footprint, demonstrating extended potential at scale.
Based on classic optimization theory, we also provide a convergence guarantee for our solver to ensure theoretical soundness alongside empirical findings.

The full extent of its limits is not yet fully understood.
However, our work clearly demonstrates significant potential for further advancements in LLM pruning.
We believe that this finding calls for a renewed focus on alternative strategies that more faithfully preserve model fidelity, which could include better ways to exchange efficiency for performance, providing practitioners with a wider range of options.

\begin{figure*}[t]
    \centering
    \includegraphics[width=0.55\textwidth]{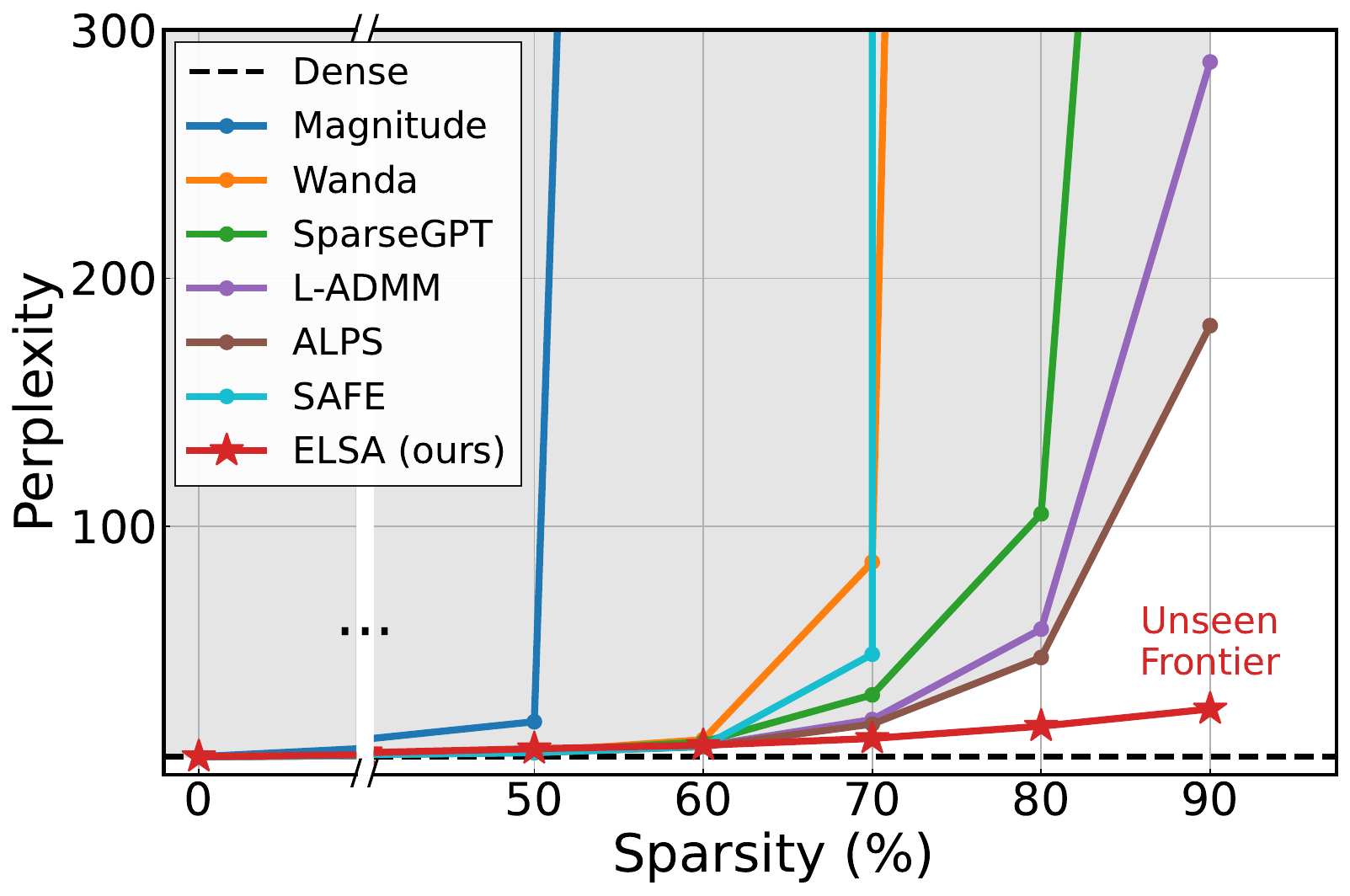}
    \caption{
        Perplexity ($\downarrow$) vs. Sparsity ($\uparrow$) curves for different pruning methods;
        it is measured on the C4 dataset for pruned LLaMA-2-7B models.
        While existing methods start to fail as sparsity increases, our approach (\ours) stays stable without losing much performance, revealing the unseen frontier.
        Previously it was considered nearly impossible to achieve such high sparsity for LLMs or go beyond the ``sparsity wall'' formed around 50-60\% sparsity levels.
        The same trend is observed consistently across different architectures and scales as we will show in \Cref{sec:experiments}--\textsc{Experiments}.
    }
    \label{fig:teaser}
\end{figure*}

% This motivates us to design \gpa, a pruner designed to well-optimize the original, global objective based on alternating direction methods of multipliers (ADMM) \citep{boyd2011distributed}.
% Our theoretical and empirical results demonstrate that \gpa is highly effective at LLM pruning, especially in extreme sparsity.
% These findings highlight the foundational validity of the divide-and-conquer paradigm into question. Our main contributions are threefold:

% \begin{itemize}
%     \item \textbf{A principled global pruning framework.}  
%     We propose \gpa, an ADMM-based constrained optimization method that directly enforces sparsity while optimizing the global objective, with theoretical convergence guarantees under mild assumptions.  
    
%     \item \textbf{Clarifying the limits of layer-wise pruning.}  
%     We empirically demonstrate that divide-and-conquer (\dnc) pruning collapses at high sparsity (70--90\%) and connect this behavior to compounding error, showing its misalignment with the global compression objective.  
    
%     \item \textbf{Revisiting LLM pruning practices.}  
%     \gpa achieves state-of-the-art results, especially at high sparsity where existing methods fail, highlighting the need to reconsider the current divide-and-conquer paradigm.  
% \end{itemize}

\section{Problem statement}
\label{sec:Problem statement}
% \subsection{The ``sparsity wall" in LLM pruning}
% The long-standing field of network pruning, aimed at enhancing model efficiency \citep{lecun1989optimal, han2015learning}, has recently seen significant progress in its application to large language models (LLMs) \citep{frantar2023sparsegpt,sunsimple,boza2024fast, liu2025proxsparse}.

The long-standing research of neural network pruning, aimed at enhancing the efficiency of large models \citep{lecun1989optimal, han2015learning}, has recently made significant progress in its application to LLMs \citep{frantar2023sparsegpt,sunsimple,boza2024fast, liu2025proxsparse}.
% For instance, the representative method SparseGPT \citep{frantar2023sparsegpt} demonstrated that LLMs can be pruned up to 50\% sparsity with negligible performance degradation.
% This is achieved by solving a series of layer-wise linear systems and stitching the resulting solutions. 
% Subsequent works have further advanced these results; for instance, Wanda \citep{sunsimple} introduced activation-aware saliency metrics, while other approaches leverage powerful constrained optimization formulations \citep{boza2024fast, liu2025proxsparse}.
While effective, these methods decline sharply and fail to maintain performance beyond a moderate level of sparsity around 50-60\%. 
For example, the recent study of \citet{zhang2024dynamic} to evaluate these methods report that their performance begins to collapse after 70\% sparsity.
This deterioration is also evident in other recent works that, notwithstanding the relative advantage over existing methods, the majority still suffer from severely degraded performance in high-sparsity regimes, with perplexity often increased more than an order of magnitude \citep{boza2024fast, meng2024alps}. 
In fact, this stands in stark contrast to historical precedents, where extreme sparsity of say 90\% or higher was commonly achieved \citep{frankle2018lottery,lee2019snip}.
% \citep{peste2023cram, lee2025safe, lee2025dynamic}. 
Consequently, researchers has begun to theorize the underlying causes, attributing the failure to compounding layer-wise errors and the explosion of reconstruction error \citep{shin2024rethinking, huang2025determining}.

These findings have collectively fostered a narrative that achieving high sparsity in language models is an illusional goal. 
We argue, however, that this ``sparsity wall'' is perhaps not an inherent limitation but rather an artifact of ill-defined problem formulation.

% \subsection{Deconstructing the layer-wise pruning paradigm}
To analyze, let us begin by showing that pruning can be formulated most generally as a constrained optimization problem as follows:
\begin{equation}\label{eq:pruning}
    x^\star = \argmin \ f(x) \quad \text{subject to} \quad \|x\|_0 \leq k    
\end{equation}
where $x \in \R^d$ refers to the optimization variable (\ie, parameters of a neural network), $f$ denotes the minimization objective (\eg, cross-entropy loss for next token prediction), and $k$ is the number of parameters to preserve after pruning.
\Ie, the successful processing of (\ref{eq:pruning}) will yield a solution $x^\star$ that is sparse and keeps prediction performance.

However, the majority of LLM pruning methods takes an approach of the following form:
\begin{equation}
\label{eq:rem}
    x^\star = \{x_i^\star \ \ \texttt{for} \ \ i=1,\dots,L \} \quad \text{where} \quad 
    x_i^\star = \argmin \ \tilde{f}(x_i) \quad \text{subject to} \quad \|x_i\|_0 \leq k_i
\end{equation}
where $L$ refers to the number of some modularized parts of the network model--most typically layers--and $\tilde{f}$ denotes a module-wise surrogate objective that measures reconstruction error;
precisely, the reconstruction error here is defined to be
\begin{equation}
\label{eq:re}
    \tilde{f} \coloneqq \E_{\mathcal{D}} \| {\bar{x}_i}^\top g(x_{i-1}; \mathcal{D}) - x_i^\top g(x_{i-1}; \mathcal{D}) \|^2
\end{equation}
% $\tilde{f} \coloneqq \E_{\mathcal{D}} \| {\bar{x}_i}^\top g(x_{i-1}; \mathcal{D}) - x_i^\top g(x_{i-1}; \mathcal{D}) \|^2$,
where $g(x_{i-1};\cdot)$ and $\bar{x}$ denote the activations of the previous layer and the $i$-th layer of the pre-trained dense model, respectively, and $\mathcal{D}$ refers to some calibration data.
Thus, the model is split into submodels, and each submodel is pruned so as to match or reconstruct the predictions of the dense counterpart on some data, sequentially until the last submodel.
The solution is then obtained by simply stacking these sparse submodels.

We posit that this approach (\ref{eq:rem}), so-called layer-wise reconstruction error minimization, introduces non-trivial and potentially critical limitations.
Specifically, we highlight three potential pitfalls: (i) errors from approximate layer-wise solutions, (ii) suboptimality in model-wide reconstruction, and (iii) the surrogacy in the objective.
We elaborate these as below.

First of all, it is hard to solve (\ref{eq:rem}) exactly without errors, in other words, the distance (\ref{eq:re}) cannot be  zero realistically.
This is due to the high cost of exactly solving sparse linear regression \citep{natarajan1995sparse}.
In fact, this leads to layer-wise solvers relying on saliency-based heuristics to find approximate solutions \citep{frantar2023sparsegpt, sunsimple, meng2024alps}.
Without zero layer-wise reconstruction errors, even small errors from each layer can compound into large overall errors, which has been observed to pose non-trivial harm to performance \citep{shin2024rethinking, huang2025determining}.

Also, its sequential, layer-wise design is naturally restrictive, potentially introducing suboptimality.
By enforcing the layer-wise features to match those of a pre-trained network, it effectively restricts the search space of the potential solutions, even though no guarantee exists that the optimal sparse model would necessarily respect this requirement.
Further concern stems from its independent and sequential nature; the layers are never jointly optimized, and notably, earlier layers will remain fixed even when subsequent layers change regardless of the potential suboptimality it introduces.

Lastly—and perhaps quite fundamentally—its reliance on a surrogate objective $\tilde{f}$ implies that one cannot expect to obtain a solution on (\ref{eq:pruning}) even after perfectly solving (\ref{eq:rem}).
% This stands in direct opposition to the underlying goal of achieving a perfect, zero error solution on (\ref{eq:rem}), whereas, in reality, it may simply lead to overfit on some calibration data $\mathcal{D}$, which may have little or even no relationship with preserving the true language modeling capabilities.
This stands in direct opposition to the underlying goal of achieving a perfect, zero error solution on (\ref{eq:rem}), whereas, in reality, it may simply lead to overfitting, failing the true objective (\ref{eq:pruning}) of preserving the language modeling capabilities.
We expect these core issues to act as a barrier as we seek higher sparsity levels.

\section{Method}

We propose \ours (\underline{E}xtreme \underline{L}LM sparsity via \underline{S}urrogate-free \underline{A}DMM) to directly solve (\ref{eq:pruning}). %, posited to be the key ingredient missing from current practice for attaining high sparsity in LLMs. 
We ground our approach in optimization from both first-principle and advanced techniques in order to better ensure that (\ref{eq:pruning}) is properly solved while enhancing effectiveness specifically for LLMs.
% Specifically, we develop advanced variants of the alternating Lagrangian method of multipliers (ADMM) algorithm to (1) better adapt to the geometry of the LLM training objective during the projection step and (2) improve scalability with low precision states for ADMM and the optimizer for the $x$-minimization.

\subsection{Surrogate-free LLM sparsification via ADMM}
We solve (\ref{eq:pruning}) using the alternating direction method of multipliers (ADMM, \citet{boyd2011distributed}), a strategy involving variable splitting to decouple the intractable sparsity constraint $\mathcal{S} = \{ v \in \mathbb{R}^d \mid \|v\|_0 \leq k \}$ from the training objective. 
This is done by introducing an auxiliary variable $z$ in the following manner:
\begin{equation} \label{eq:admm_split}
\min_{x,z} f(x) + I_{\mathcal{S}}(z) \quad \text{s.t.} \quad x=z,
\end{equation}
where $I_{\mathcal{S}}(z)$ is the indicator function for the set $\mathcal{S}$:
\begin{equation}
    I_{\mathcal{S}}(z) := 
    \begin{cases}
        0 & \text{if } z \in \mathcal{S} \\
        \infty & \text{otherwise.}
    \end{cases}
\end{equation}
In turn, we keep $x$ constrained to be equal to $z$.
This allows us to handle the model training and the sparsity satisfaction somewhat separated, making both much easier to handle.

To solve for this new formulation, the augmented Lagrangian can be used:
\begin{equation}
    \label{eq:augmented_lagrangian}
    \mathcal{L}_{\lambda}(x, z, u) = f(x) + I_{\mathcal{S}}(z) + \frac{\lambda}{2} \|x - z + u\|_2^2 - \frac{\lambda}{2}\|u\|_2^2 \ ,
\end{equation}
where $\lambda$ is the hyperparameter for adjusting the strength of the proximal penalty, and $u$ is a scaled dual variable.
ADMM solves this by alternating between minimizing the augmented Lagrangian over the primal variables ($x,z$) and performing a dual ascent step on $u$.
This decomposes the problem into three manageable subproblems that are iterated until convergence:
\begin{align}
    x^{t+1} &= \argmin_x \left( f(x) + \frac{\lambda}{2}\|x-z^t+u^t\|^2_2 \right) \label{eq:x_update} \ , \\
    z^{t+1} &= \argmin_{z\in \mathcal{S}} \frac{\lambda}{2}\|x-z^t+u^t\|^2_2  = \Pi_{\mathcal{S}}(x^{t+1}+u^t) \ , \label{eq:z_update} \\
    u^{t+1} &=u^t + x^{t+1} - z^{t+1} \ . \label{eq:u_update}
\end{align}
The $x$-update (\ref{eq:x_update}) accounts for minimizing the training objective, and is iteratively minimized while $x$ is pushed closer to the sparse $z$.
The $z$-update (\ref{eq:z_update}) can be expressed as the projection $\Pi_{\mathcal{S}}(x^{t+1}+u^t)$.
Here, the objective associated with its $\mathcal{S}$ is simplified to minimizing the Euclidean distance from $x^{t+1}+u^t$, effectively replacing the complex, non-convex $f$ with a tractable, convex quadratic function.
As a result, this has an exact closed-form solution computable by zeroing out the $(d-k)$-entries with the smallest magnitude \citep{lee2025safe}.
Finally, the scaled dual variable $u$ is updated in (\ref{eq:u_update}) to maximize the augmented Lagrangian via a single step of gradient ascent.%, accumulating the residual between $x$ and $z$.

% \subsection{Objective-aware generalized projection}
% \label{sec:generalized_projection}

\subsection{Objective-aware projection}
\label{subsec:generalized projection}
% The standard projection in the $z$-update, $\Pi_{\mathcal{S}}(\cdot)$, is a simple top-$k$ magnitude selection, where...  potentially disrupting the $x$-update...
% To address this issue, we introduce a generalized projection that incorporates information from the global loss $ f(x)$.
Closely inspecting the projection step in the $z$-update (\ref{eq:z_update}), one can see that the Euclidean distance is far too removed from $f$.
Thus, it is reasonable to expect that the sparse parameters obtained in $z$ may differ considerably from the actual sparse optima of $f$.
% Its consequences are quite straightforward: it is essentially magnitude pruning—a most basic pruning strategy that typically serves as the weakest baseline in practice.
% We indeed observe a pronounced spike in the training loss when running a naive ADMM in our experiments, indicating that $z$ does not adequately respect the training objective, making it difficult to push $x$ toward this $z$ with confidence without impairing performance.

This motivates us to align the projection step with $f$ by modifying its objective into the following quadratic:
\begin{equation}
\label{eq:projection_step}
z^{t+1} = \operatorname*{arg\,min}_{z \in \mathcal{S}} 
\frac{1}{2}(z-(x^{t+1} + u^t))^\top 
\mathbf{H} \,(z-(x^{t+1}+u^t)),
\end{equation}
where $\mathbf{H}$ is the Hessian of $f$.
% This roughly corresponds to using a quadratic approximation of $f$;
Equivalently, we project in the $\mathbf{H}$ induced norm, aligning the step with the second-order geometry of $f$.
% more precisely, this new objective-aware objective is constructed to capture how parameter changes translate into changes in $f$ up to its second-order Taylor expansion.
Placed once again in the context of pruning research, its advantages would be akin to those of the family of approaches based on the Optimal Brain Surgeon algorithm \citep{lecun1989optimal}.
% a far more advanced method for sparsifying neural networks compared to magnitude pruning.

In practice, two approximations are introduced.
We notice that the procedural simplicity in the Euclidean case stems from the objective being separable across entries.
We found that using $\text{Diag}(\mathbf{H})$ allows us to retain this simplicity while still keeping the benefits by zeroing the entries with the smallest contribution to the objective rather than by their magnitudes.
Also, we employ the Gauss-Newton approximation of the Hessian or the empirical Fisher information matrix $\hat{\mathbf{F}}$, which allows us to obtain a good approximation of the Hessian only by the outer products of the gradients \citep{martens2020new}.
The results of these can be summarized into the following formula:
\begin{equation}
\label{eq:projection_step}
    z^{t+1} = \operatorname*{arg\,min}_{z \in \mathcal{S}} 
    \sum_{i\leq d} \mathbf{\hat{F}}_{ii} \,(z_i-(x_i^{t+1}+u_i^t))^2,
\end{equation}
where each coordinate $i$ contributes independently to this new loss function.
Luckily, the standard Adam optimizer has already made $\hat{\mathbf{F}}$ available for free via its second-moment estimates, requiring no additional cost in implementing this enhancement \citep{li2025fishers}.
Overall, this tailors our algorithm \ours to better adapt to the complex objective of LLMs, and in a way that incurs negligible additional cost.

\subsection{Scalable ADMM via low-precision states}
\label{subsec:elsal}
% ADMM incur costs, and we cannot take naive action without its convergence guarantees.
We further enhance scalability by proposing \oursq.
Here, we rely on two core operations: a quantization operation, $\mathcal{Q}$, that maps high-precision tensors to a compact low-precision representation, and a dequantization operation, $\mathcal{R}$, that rematerializes them.

Formally, for a high-precision tensor $z \in \mathbb{R}^d$, the $\mathcal{Q}$ operation produces a storable pair $(z_q, s)$ consisting of a quantized tensor and a scale:
\begin{equation}
\label{eq:quant_op}
\mathcal{Q}(z) \triangleq (z_q, s), \quad \text{where } \enspace s = \max(|z|)/v_{\max} \enspace \text{and} \enspace z_q = \text{round}\left(z/s\right).
% \mathcal{Q}(z) \triangleq (z_q, s), \quad \text{where } \enspace s = \frac{\max(|z|)}{v_{\max}} \enspace \text{and} \enspace z_q = \text{round}\left(\frac{z}{s}\right).
\end{equation}
Here, $v_{\max}$ is the maximum representable absolute value of the target data type (\eg, $127$ for signed \texttt{INT8}).
Conversely, the $\mathcal{R}$ operation rematerializes the high-precision tensor from the stored pair:
\begin{equation}
\label{eq:dequant_op}
\mathcal{R}(z_q, s) \triangleq s \cdot z_q.
\end{equation}
These operations are applied in a cycle to manage the auxiliary variables.
After a high-precision update yields an intermediate state, for instance $z^{t+1} = \Pi_\mathcal{S}(x^{t+1}+u^t)$, it is quantized for efficient storage: $(z_q^{t+1}, s^{t+1}) = \mathcal{Q}(z^{t+1})$. 
This transition yields substantial memory savings; for instance, storing a state in \texttt{FP8} ($8$ bits) reduces the memory footprint by $4\times$ compared to the standard \texttt{FP32} representation ($32$ bits).
The overhead from the scale factor is negligible, as typically only a single $32$-bit scale value is stored for the entire tensor.
For the subsequent computation, the state is first rematerialized to high precision: $\hat{z}^{t+1} = \mathcal{R}(z_q^{t+1}, s^{t+1})$.

This quant-dequant cycle, which bridges low-precision storage with high-precision updates via a dynamic, data-aware scale, is a general and established principle in low-precision deep learning \citep{gholami2022survey}.
The specific definitions in (\ref{eq:quant_op}) can be adapted for various formats, including both $8$-bit integers (\texttt{INT8}) \citep{jacob2018quantization} and modern floating-point types like \texttt{FP8}, representing a cornerstone of efficient numerical methods \citep{micikevicius2022fp8}.

However, this introduces nontrivial changes into the algorithm, and thus, the guarantees of ADMM do not automatically extend.
We therefore establish a proof to demonstrate that \oursq, alongside with \ours, will converge to the solution of (\ref{eq:pruning}) in the following section.

% To enhance the scalability and memory efficiency, we introduce a quantization function, $\quant (\cdot)$, which maps high-precision values to a low-precision format. 
% This function is applied to the auxiliary variables and their corresponding optimizer states during the update steps:
% \begin{equation}
% z^{t+1} = \quant \left( \Pi_{\mathcal{S}}(x^{t+1}+u^t) \right), \quad u^{t+1} = \quant \left( u^t + x^{t+1} - z^{t+1} \right)
% \end{equation}
% This approach makes a global pruning method feasible even for billion-parameter models.
% Crucially, we demonstrate in our theoretical analysis (\cref{sec:gpa_conv}) that this quantized variant maintains its convergence guarantees. 

% \section{Theoretical analysis}

% \subsection{Limitations of local reconstruction}
% \begin{itemize}
%     \item suboptimality of local REM on global REM problem
% \end{itemize}
\vspace{-0.5em}
\section{Convergence analysis}
\label{sec:gpa_conv}
We establish theoretical convergence for both \ours and \oursq to support their reliability in directly solving (\ref{eq:pruning}).
% Precisely, we review established ADMM convergence theories \citep{boyd2011distributed, huang21ADMMQ} to justify \ours and develop new analyses to guarantee the convergence of \oursq.
Formally, we assume the following:

\begin{assumption} \label{assump:lowbound}
    (Lower bounded on constraint)  The function $ f $ is lower bounded on $ \mathcal{S} $. That is, there exists a constant $ f_{\min} := \min_{a \in \mathcal{S}} f(a) $ and $ f_{\min} > -\infty $.
\end{assumption}
\begin{assumption} \label{assump:smooth}
    ($\beta$-smoothness) The function $ f $ is differentiable, and its gradient is $\beta$-smooth. That is, $\|\nabla f(x) - \nabla f(y)\| \leq \beta \|x - y\|$
\end{assumption}
\begin{assumption} \label{assump:weakcvx}
    ($\mu$-weak convexity) There exists a constant $ \mu \geq 0 $ such that $ f $ is $ \mu $-weakly convex. i.e., $f(x) + \frac{\mu}{2} \|x\|^2$ is convex.
\end{assumption}

Also, we rely on the notion of $\lambda$-stationarity \citep{huang21ADMMQ}:
% Given these assumptions, we rely on the notion of $\lambda$-stationarity \citep{huang21ADMMQ}:
\begin{definition} \label{def:stationary}
    ($\lambda$-stationary point) We say a point $ \bar{x} $ is a $\lambda$-stationary point of the optimization problem (\ref{eq:pruning}) if $\bar{x} \in \arg \min _{x\in \mathcal{S}}\left\|x-\left(\bar{x}-\lambda^{-1} \nabla f(\bar{x})\right)\right\|,$
\end{definition}
\ie, the point $\bar{x}$ cannot be locally improved using projected gradient descent with step-size $\lambda^{-1}$.
% This can be understood as  

Given these, we present the convergence of \ours and \oursq as follows:

% \subsubsection{Convergence of \ours} \label{sec:gpa_conv_gpa}
% According to \citet{huang21ADMMQ}, we expect \ours{} 

\begin{corollary} \label{cor:gpa_conv}
    (Convergence of \ours)
    Suppose that Assumptions \ref{assump:lowbound}-\ref{assump:weakcvx} hold.
    Assume further that $\lambda$ is chosen large enough so that $\lambda^{-1}\beta^2-(\lambda-\mu)/2<0$.
    Let $(\bar{x}, \bar{z}, \bar{u})$ be a limit point of \ours algorithm.
    Then $\bar x$ is a $\lambda$-stationary point of (\ref{eq:pruning}).
\end{corollary}

% \subsubsection{Convergence of \oursq} \label{sec:gpa_conv_gpaq}

\begin{theorem} \label{thm:gpaq_conv}
(Convergence of \oursq)
Suppose that Assumptions \ref{assump:lowbound}-\ref{assump:weakcvx} hold.
Also assume that the iterates of  \oursq are bounded, and the constant $\lambda$ and $\gamma$ are chosen such that %$\gamma\leq \frac{1}{2}$ and
\begin{equation}
    \nonumber
     \frac{\beta^2}{\lambda} + \frac{\beta(\lambda + \beta) \gamma}{\lambda}  + \frac{{\gamma}^2(\lambda+\beta)}{2}-\frac{(1-{\gamma})^2(\lambda-\mu)}{2} < 0.
     \label{eq: gpaq_cond_gamma}
\end{equation}
Then, for any limit point $(\bar{x}, \bar{z}, \bar{u})$ of the iterates, $\xb$ is a $\lambda$--stationary point of (\ref{eq:pruning}).
\end{theorem}

This demonstrates that \ours and \oursq converge to the stationary point of the sparsity-constrained optimization problem (\ref{eq:pruning}).
The detailed proof for \oursq is provided in \cref{app: proof_qgpa}.
% We note that, while the technical contributions of our analysis might be considered modest, \ours is built on a theoretically rigorous foundation, unlike many other pruning techniques that often rely primarily on ad-hoc intuitions.

\begin{figure}[t]
    \centering
    \includegraphics[width=0.98\linewidth]{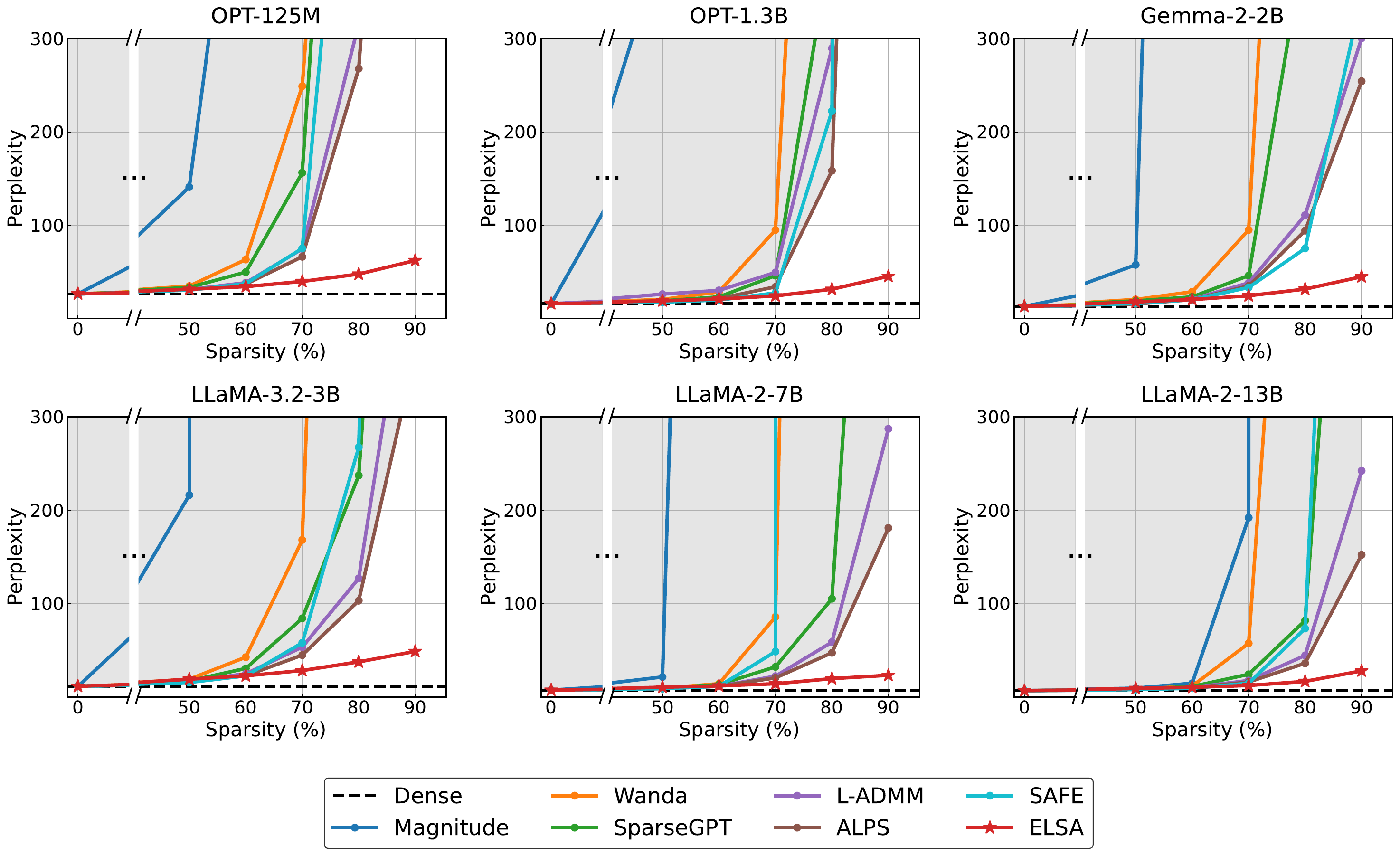}
    \caption{
        Perplexity vs. Sparsity plots for different models and scales.
        \ours preserves much lower perplexity at high sparsity compared to other methods, consistently across a wide range of settings, showing its advantage and robustness.
        All numerical results are provided in \cref{Appendix:Additional results}.}
    \label{fig:main_figure}
    % \vspace{-1.7em}
\end{figure}

\section{Experiments}
\label{sec:experiments}

We present a series of concrete experiments to validate \ours in this section.
Specifically, we show that \ours
(i) effectively prunes models to extreme high sparsity levels across a wide range of models and scales (\Cref{subsec:main_results}),
(ii) can further push the limits to extreme sparsity (up to 99\%) (\Cref{subsec:extreme_sparsity}), (iii) can accelerate inference while reducing memory requirements (\Cref{subsec:inference_acceleration}), and (iv) scales efficiently to large models up to 27B (\Cref{subsec:scaling_results}).
Lastly, we analyze the cost of \ours compared to existing methods at (\Cref{subsec:cost_analysis}).
We also provide extension to other sparsity patterns such as N:M semi-structured sparsity or non-uniform sparsity, and an ablation study on the choice of objective functions and generalized projection on \Cref{Appendix:Additional results}.
% (ii) scales efficiently to large models up to 27B (\Cref{subsec:scaling_results}), and (iii) adapts to other sparsity patterns such as N:M semi-structured sparsity or non-uniform sparsity found by evolutionary strategies (\Cref{subsec:other_sparsity}).
% We also provide an extablation study on the choice of objective functions and generalized projection (\Cref{subsec:ablation}).
% Collectively, these results establish the effectiveness, scalability, and generalizability of \ours.

We compare \ours to the following methods: Magnitude \citep{han2015learning}, SparseGPT \citep{frantar2023sparsegpt}, Wanda \citep{sunsimple}, ALPS \citep{meng2024alps}, L-ADMM (Layer-wise ADMM) \citep{boza2024fast}, SAFE \citep{lee2025safe}, and SparseLLM \citep{bai2024sparsellm}.
These methods are applied to four different architectures including OPT \citep{zhang2022opt}, Gemma-2 \citep{team2024gemma}, and LLaMA-2/3 \citep{touvron2023llama,grattafiori2024llama3herdmodels} across a wide range of scales from 125M to 27B.
We report perplexity and zero-shot prediction accuracy of pruned models.
All experiment settings can be found in \cref{Appendix: Experimental details}.
The source code to reproduce results will be made available at \texttt{https://github.com/log-postech/elsa}.
% All our source code to reproduce the results will be released upon publication.

% We assess model performance on two fronts: language modeling capability and downstream task generalization. 
% For language modeling, we measure perplexity on the test dataset of C4 (\texttt{C4}) \citep{raffel2020exploring}/Wikitext2 (\texttt{Wiki}) \citep{merity2017pointer}, and zero-shot accuracy of 7 different downstream tasks using \texttt{lm-eval-harness} \citep{eval-harness}.
% Details for training, evaluation, and tasks can be found at \cref{Appendix:details_5.1}.

\begin{figure*}[t]
    \centering
    \includegraphics[width=0.45\textwidth]{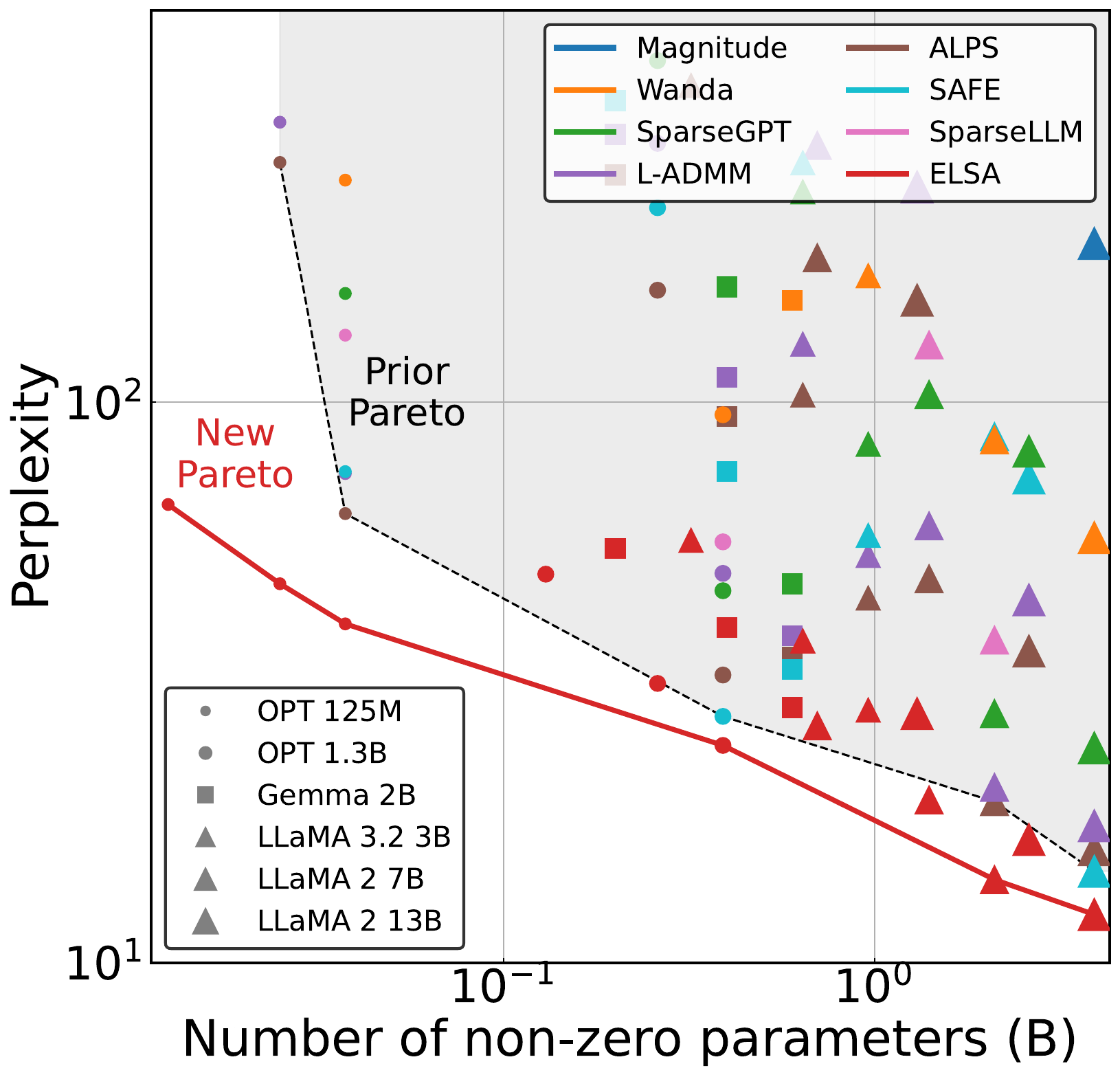}
    \caption{
        Pareto optimality of \ours compared to prior works in terms of 
        perplexity vs. number of non-zero parameters.
        \ours displays its greater optimality across a broad spectrum of effective scales.
        % $((1-s) \times D)$ across model scales (shapes) and pruning methods (colors) at 70--90\% sparsity.
        % While prior approaches saturate at the existing frontier (dashed line), \ours consistently pushes beyond this boundary, establishing a new frontier and demonstrating robust performance even under extreme sparsity.
    }
    \label{fig:frontier}
\end{figure*}
% \begin{wrapfigure}{r}{0.33\textwidth} 
%     \centering
%     \vspace{-4em}
%     \includegraphics[width=0.33\textwidth]{fig/figure_3.pdf} 
%     \caption{
%         Pareto optimality of \ours compared to prior works in terms of 
%         perplexity vs. number of non-zero parameters.
%         % \ours displays its greater optimality across a broad spectrum of effective scales.
%         }
%     \label{fig:frontier}
%     \vspace{-4em}
% \end{wrapfigure}

\subsection{Main results}
\label{subsec:main_results}
% We present our main experimental results, comparing \ours against 7 different pruning methods across 6 different models.
\cref{fig:main_figure} reports C4 perplexity for various models across different sparsity levels from 50\% to 90\%.  
Existing methods deteriorate rapidly beyond 70\%; for instance, SparseGPT on OPT-125M rises from 49.83 at 60\% sparsity to over 1,000 at 80\%.  
In contrast, \ours remains stable, increasing only from 42.99 to 47.45 over the same range, and at 80\% sparsity matches the perplexity of SparseGPT at 60\%.  
This robustness holds across scales: on LLaMA-2-13B at 90\% sparsity, \ours achieves 27.84 perplexity, while most existing methods exceed the hundreds.
\cref{fig:frontier} further highlights this trend by plotting perplexity against the effective number of non-zero parameters.  
\ours consistently sets the new Pareto frontier across scales, underscoring its robustness in extreme sparsity regimes.

This extends to downstream task performance, as shown in \cref{fig:main_figure-2}.  
Each radar plot reports per-task accuracy at high sparsity (70–90\%), with the enclosed area reflecting the average accuracy across tasks.  
At 70\% sparsity, \ours is competitive with leading methods, but a clear gap emerges as sparsity increases.  
From 70\% to 80\% sparsity, other methods lose 10–20\%p accuracy on tasks such as Winogrande and ARC-E, while \ours degrades by less than half as much. At 90\%, most methods collapse, whereas \ours retains the highest accuracy on 6 out of 7 tasks, with an average margin of 6.06\%p.
This demonstrates that \ours maintains generalization far better than existing methods at high sparsity.
We believe that these results collectively establish the effectiveness of \ours for high sparsity.

\begin{figure}[t]
    \centering
    \includegraphics[width=0.9\linewidth]{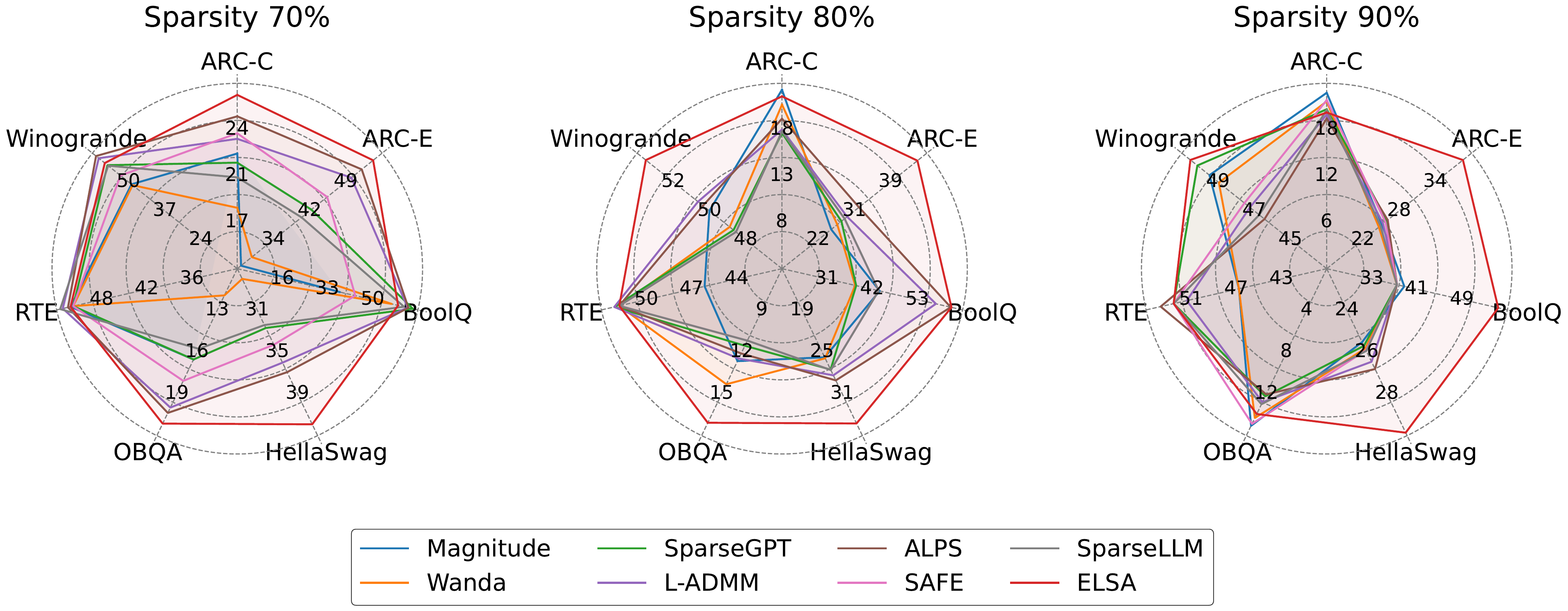}
    \caption{Zero-shot accuracy of pruned LLaMA-2-7B models.
    \ours outperforms other methods for most tasks, with the performance gap widening as sparsity increases, highlighting its strong generalization capability.
    Full numerical results are provided in \cref{tab:llama7b-zeroshot} of Appendix~\ref{Appendix:Additional results}.
    }
    \label{fig:main_figure-2}
    % \vspace{-1.8em}
\end{figure}

\begin{table}[t]
\centering
\caption{Memory savings and inference accelerations of \ours on LLaMA-2-7B.}
\label{tab:inference_acceleration}
\setlength{\tabcolsep}{3.0pt}
\renewcommand{\arraystretch}{1.05}
\small
\begin{tabular}{l r r l r l r l r l}
\toprule
 & \textbf{Dense}
 & \multicolumn{2}{c}{\textbf{50\%}}
 & \multicolumn{2}{c}{\textbf{70\%}}
 & \multicolumn{2}{c}{\textbf{90\%}}
 & \multicolumn{2}{c}{\textbf{95\%}} \\
\midrule
\textbf{Latency (s)}
& 1.84
& 1.37 & ($\times$1.34)
& 0.95 & ($\times$1.94)
& \textbf{0.72} & \textbf{($\times$2.50)}
& \textbf{0.46} & \textbf{($\times$4.00)} \\

\textbf{Tokens/s}
& 54.47
& 72.81 & ($\times$1.33)
& 104.89 & ($\times$1.93)
& \textbf{139.47} & \textbf{($\times$2.56)}
& \textbf{216.98} & \textbf{($\times$3.98)} \\

\textbf{Memory (MB)}
& 13596
& 8870 & ($\times$1.53)
& 5603 & ($\times$2.42)
& \textbf{2918} & \textbf{($\times$4.60)}
& \textbf{1743} & \textbf{($\times$7.80)} \\
\bottomrule
\end{tabular}
\end{table}

\begin{wraptable}{r}{0.35\textwidth}
    \centering
    \scriptsize
    \vspace{-1.8em}
    \caption{Perplexity on Wiki/C4 at extreme sparsity.}
    \label{tab:extreme_ppl}
    \setlength{\tabcolsep}{4pt}
    \renewcommand{\arraystretch}{1.05}
    \scriptsize
    \begin{tabular}{c l c c}
    \toprule
    \textbf{Sparsity} & \textbf{Method} & \textbf{Wiki} & \textbf{C4} \\
    \midrule
    \multirow{3}{*}{0.90}
    & Wanda + LoRA            & 92.66  & 65.56 \\
    & Wanda + Full & 42.40  & 34.87 \\
    & \ours           & \textbf{26.97} & \textbf{23.14} \\
    \midrule
    \multirow{3}{*}{0.95}
    & Wanda + LoRA            & 371.0  & 143.0 \\
    & Wanda + Full & 84.30  & 53.62 \\
    & \ours           & \textbf{38.91} & \textbf{28.39} \\
    \midrule
    \multirow{3}{*}{0.99}
    & Wanda + LoRA            & 588.3  & 247.5 \\
    & Wanda + Full & 146.37 & 71.64 \\
    & \ours           & \textbf{55.94} & \textbf{40.10} \\
    \bottomrule
    \end{tabular}
    % \vspace{-2.0em}
\end{wraptable}

\subsection{Towards extreme sparsity}
\label{subsec:extreme_sparsity}
To assess whether \ours remains effective under extreme sparsity, we further evaluate LLaMA 2-7B at 95\% and 99\% sparsity. 
We additionally consider retraining after pruning with Wanda \citep{sunsimple} for the baseline,
and the experimental details can be found at \cref{Appendix:5.2}.

As shown in Table~\ref{tab:extreme_ppl}, \ours consistently achieves substantially lower perplexity on both WikiText and C4 dataset, and the margin increases as sparsity becomes more extreme. 
Notably, at 99\% sparsity, \ours remains stable ($55.94/40.10$), whereas Wanda+full fine-tuning degrades ($146.37/71.64$) and Wanda+LoRA collapses ($588.3/247.5$). 
These results show that \ours remains stable even in this extreme regime, while simple heuristics fail to preserve performance, highlighting the advantage of \ours’s more principled approach.

\subsection{Realizing benefits with extreme sparsity}
\label{subsec:inference_acceleration}

% \begin{table}[t]
% \centering
% \caption{Memory savings and inference accelerations of \ours on LLaMA-2-7B.}
% \label{tab:inference_acceleration}
% \setlength{\tabcolsep}{3.2pt}      % tighter columns
% \renewcommand{\arraystretch}{1.05} % tighter rows
% \small
% \begin{tabular}{llllll}
% \toprule
%  & \textbf{Dense} & \textbf{50\%} & \textbf{70\%} & \textbf{90\%} & \textbf{95\%} \\
% \midrule
% \textbf{Latency (s)}  & 1.84 & 1.37 ($\times$1.34) & 0.95 ($\times$1.94) & \textbf{0.72 ($\times$2.50)} & \textbf{0.46 ($\times$4.00)} \\
% \textbf{Tokens/s}     & 54.47 & 72.81 ($\times$1.33) & 104.89 ($\times$1.93) & \textbf{139.47 ($\times$2.56)} & \textbf{216.98 ($\times$3.98)} \\
% \textbf{Memory (MB)}  & 13596 & 8870 ($\times$1.53) & 5603 ($\times$2.42) & \textbf{2918 ($\times$4.60)} & \textbf{1743 ($\times$7.80)}  \\
% \bottomrule
% \end{tabular}
% \end{table}

Sparsity is practically meaningful only if it yields real deployment gains beyond reducing parameter count. 
To quantify this, we benchmark end-to-end text generation on LLaMA 2-7B using MACKO \citep{macko2025macko}, a recent Sparse-Matrix Vector multiplication (SpMV) kernel combined with memory-efficient MACKO format, supporting an acceleration and memory savings for sparse models. 
Following MACKO’s standard end-to-end protocol, we convert the sparse models produced by \ours and evaluate on a single NVIDIA RTX 3090, reporting mean latency of token generation (s), throughput (tokens/s), and memory footprint (MB).

Table~\ref{tab:inference_acceleration} shows that sparsity yields clear gains in both speed and memory, with improvements becoming pronounced from 70\% sparsity onward. 
At 90\% sparsity, \ours achieves a $2.50\times$ reduction in end-to-end latency and a $2.56\times$ increase in throughput, together with a $4.60\times$ memory reduction compared to the dense baseline. 
Pushing further to 95\% sparsity provides even stronger efficiency gains (up to $4.00\times$ latency reduction, $3.98\times$ throughput increase, and $7.80\times$ memory reduction). These results validate that the high/extreme-sparsity regime targeted by ELSA is not merely an academic frontier: it enables tangible acceleration and substantial memory savings in memory-constrained settings, particularly benefiting the decoding phase where sparse matrix--vector computation dominates.

\subsection{Scaling to large-r models}
\label{subsec:scaling_results}
\begin{wrapfigure}{r}{0.35\textwidth} 
    \centering
    \vspace{-2em}
    \includegraphics[width=0.35\textwidth]{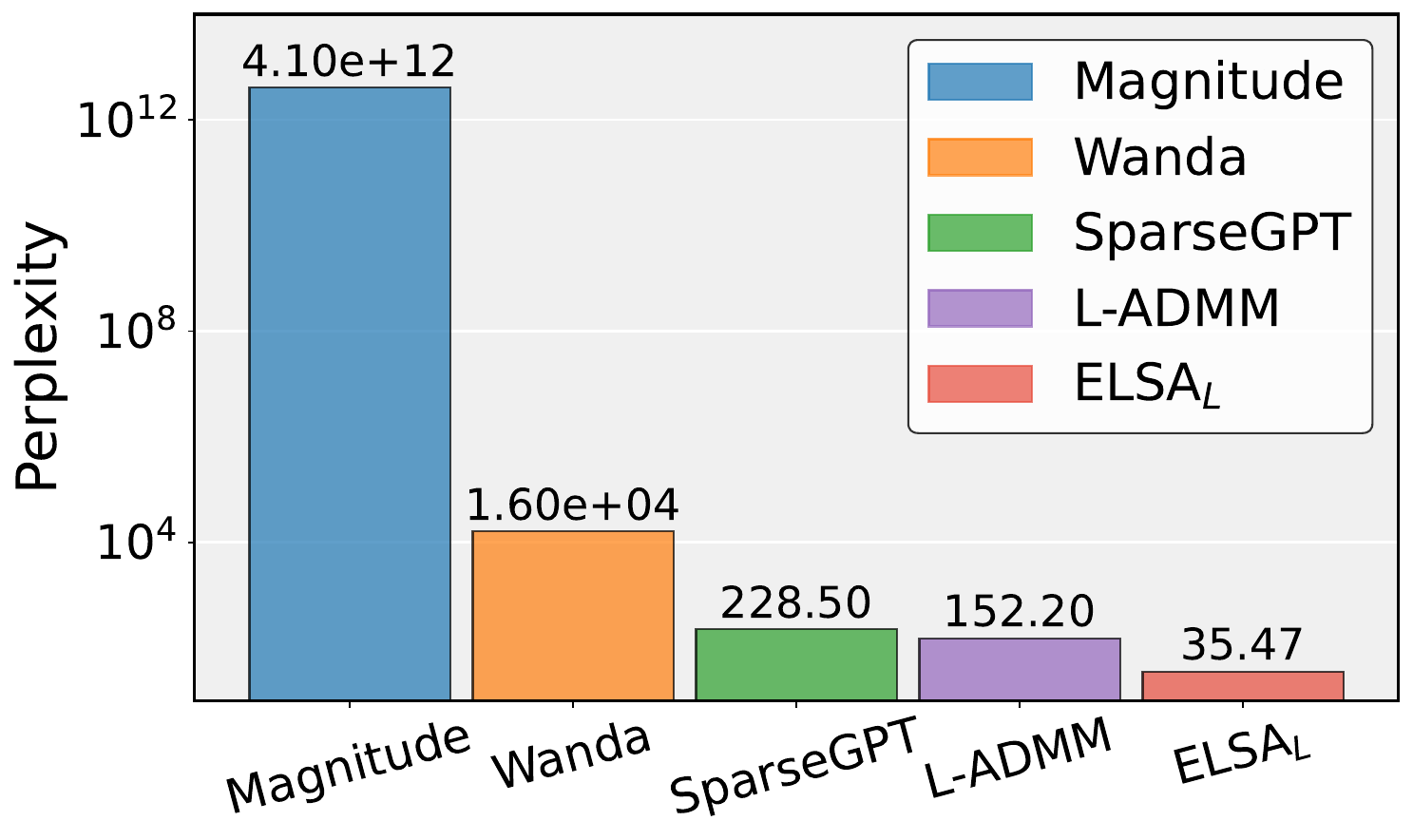} 
    \caption{
        Perplexity of Gemma-2-27B at 90\% sparsity.
        % \oursq achieves the lowest perplexity, confirming its strength.
    }
    \label{fig:gemma-27b}
    \vspace{-3em}
\end{wrapfigure}
We further validate the scalability of our approach by applying \oursq to 27B-scale (Gemma-2-27B).
Specifically, we additionally employ the low-precision optimizer $\texttt{adam8bit}$ \citep{8bit2022dettmers} for $x$-update \cref{eq:x_update}, with \oursq, where we use ($\texttt{bf16} ,\texttt{fp8}$) for auxiliary state $(u,z)$, respectively. 
This design reduces the memory footprint of required states (optimizer, auxiliary variables) by 55\% compared to \ours, enabling pruning at 27B scale under limited resources.
Additional implementation details can be found in \cref{Appendix:details_5.2}.

\cref{fig:gemma-27b} demonstrates that \oursq achieves the lowest perplexity among all compared methods, outperforming the strongest competing method by a factor of $4\times$,
supporting our main results at scale.
% These results reinforce our main finding that \ours preserves model quality even at extreme sparsity and scale.
% Additional implementation details can be found in \cref{Appendix:details_5.2}.

\subsection{Cost analysis}
\label{subsec:cost_analysis}
To assess \emph{effective efficiency} (i.e., quality achieved per pruning compute at fixed sparsity), we measure end-to-end pruning cost on identical NVIDIA A100-80GB GPUs and report GPU-hours together with perplexity for LLaMA-2-7B at 90\% sparsity (Table~\ref{tab:cost_analysis}).
One-shot methods (Wanda, SparseGPT) are the cheapest (0.16/0.25 hours with single GPU) but collapse at this sparsity level, yielding unusable perplexities ($\ge 10^4$).
The best-performing baseline ALPS improves perplexity but still remains far from usable, while requiring substantial wall-clock time (12.57 hours).
Overall, these layer-wise pruners are memory-light by design, yet additional pruning compute does not reliably translate into commensurate quality gains under high sparsity.

In contrast, \ours achieves substantially lower perplexity (26.97/23.14) with a moderate compute budget (7.12 GPU-hours), yielding a markedly better cost–quality point in this regime.
Moreover, even when augmented with matched-budget retraining (Wanda+LoRA/Full), one-shot+retrain baselines do not reach comparable quality, indicating that additional compute is more effectively converted into model quality by \ours at high sparsity.

\vspace{2em}

\begin{table}[t]
\centering
\caption{Compute cost vs.\ perplexity of different pruners on LLaMA-2-7B at 90\% sparsity.}
\label{tab:cost_analysis}
\setlength{\tabcolsep}{3.6pt}
\renewcommand{\arraystretch}{1.05}
\small
\begin{tabular}{lcccc}
\toprule
Method & Wall-clock (h) & \#GPUs & Wiki & C4 \\
\midrule
Wanda     & 0.159 & 1 & $2.0\times10^{4}$ & $1.0\times10^{4}$ \\
SparseGPT & 0.251 & 1 & 1430 & 864.5 \\
ALPS      & 12.57 & 1 & 248 & 180.9 \\
\midrule
Wanda+LoRA & 4.03 & 1 & 92.66 & 65.56 \\
Wanda+Full & 1.64 & 4 &  42.40 & 34.87 \\
\ours     & 1.78  & 4 & \textbf{26.97} & \textbf{23.14} \\
\bottomrule
\end{tabular}
\end{table}

\vspace{-0.5em}
\section{Discussion}
\label{sec:discussion}
\vspace{-0.5em}

In this work, we confront the problem of moderate sparsity in LLMs through a critical inspection into the current practice, revealing that the prevailing reliance on the sequential layer-wise reconstruction surrogate may have been constraining the path toward more extreme sparsities.
This led us to develop \ours and \oursq, enabling us to push the sparsity from 50--70\% up to 80--90\% while maintaining strong language modeling performance, where we also observe tangible deployment benefits such as inference acceleration and memory savings, and further showing that our approach remains effective even in more extreme regimes (e.g., 95--99\%).
Grounding on optimization principles ensures that our principle effectively solves the true LLM objective as is, while also facilitating the development of advanced techniques that are both theoretically sound and effective for sparsifying LLMs, which we believe were instrumental in attaining strong practical results.

Meanwhile, we remark on the memory demands associated with pruning LLMs.
In particular, we propose to reassess the widespread assumption that, given the limitations of commodity memory, the adoption of a layer-wise surrogate strategy is difficult to circumvent.
First of all, it is worth questioning whether the underlying assumption itself is too restrictive---after all, one would not typically attempt to prune an LLM without at least the resources required to run one.
Also, we raise doubts about whether the layer-wise strategy provides clear memory advantages.
Precisely, using the offloading technique allows one to optimize over the entire model with similar memory efficiency.
In fact, quite the opposite may be the case---they do not scale well with the size of calibration data, requiring the layer activations of the entire calibration data to be stored, while a single mini-batch usually suffices the surrogate-free principle.
This calls into question whether our perception of its efficiency could be somewhat inflated, requiring the need for a careful assessment of current practice and exploration of alternative strategies through a more balanced lens.

There are many promising directions to pursue for future work: (i) alternative efficiency strategies through advanced memory-efficient and derivative-free optimizers, (ii) system-level advancements in memory offloading, and (iii) extensions to advanced architecture such as Mixture-of-Experts \citep{mu2025comprehensive} and multi-modal large language models \citep{yin2024survey}.
To conclude, our work validates that the frontier of LLM sparsity can still be expanded by offering a concrete strategy supported by strong empirical evidence.
We hope it sets the stage for future breakthroughs and innovations in new directions that have thus far received relatively limited attention.

\section*{Acknowledgements}

This work was partly supported by 
the Institute of Information \& communications Technology Planning \& Evaluation (IITP) grant funded by the Korean government (MSIT)  %%%% IITP %%%%
(RS-2019-II191906, Artificial Intelligence Graduate School Program (POSTECH),% 인공지능대학원
RS-2022-II220959, (part2) Few-Shot learning of Causal Inference in Vision and Language for Decision Making),  % XAI
the National Research Foundation of Korea (NRF) grant funded by the Korean government (MSIT) %%%% NRF %%%%
( 
RS-2023-00210466, % 우수신진
% RS-2023-00265444, %%% 최초혁신 
RS-2025-02264052). %% 기초연구실 

% \section*{Impact Statement}

% This paper presents work aimed at advancing the field of Machine Learning, with the potential to influence both theoretical understanding and practical applications. While our contributions do not directly raise immediate concerns requiring specific emphasis, we acknowledge that advancements in this domain can have far-reaching societal implications.
% We will ensure ongoing discourse on the broader impact of our work in diverse contexts if the need is later recognized.

\bibliography{reference}

@article{mu2025comprehensive,
  title={A comprehensive survey of mixture-of-experts: Algorithms, theory, and applications},
  author={Mu, Siyuan and Lin, Sen},
  journal={arXiv preprint arXiv:2503.07137},
  year={2025}
}

@article{yin2024survey,
  title={A survey on multimodal large language models},
  author={Yin, Shukang and Fu, Chaoyou and Zhao, Sirui and Li, Ke and Sun, Xing and Xu, Tong and Chen, Enhong},
  journal={National Science Review},
  volume={11},
  number={12},
  pages={nwae403},
  year={2024},
  publisher={Oxford University Press}
}

@inproceedings{li2025fishers,
  title     = {Fishers for Free? Approximating the Fisher Information Matrix by Recycling the Squared Gradient Accumulator},
  author    = {Li, YuXin and Dangel, Felix and Tam, Derek and Raffel, Colin},
  booktitle = {Proceedings of the 2025 International Conference on Machine Learning (ICML)},
  year      = {2025},
  month     = {Jul},
  note      = {Poster},
  url       = {https://arxiv.org/abs/2507.18807},
  doi       = {10.48550/arXiv.2507.18807}
}

@article{lecun1989optimal,
  title={Optimal brain damage},
  author={LeCun, Yann and Denker, John and Solla, Sara},
  journal={NeurIPS},
  year={1989}
}

@article{sieberling2024evopress,
  title={Evopress: Towards optimal dynamic model compression via evolutionary search},
  author={Sieberling, Oliver and Kuznedelev, Denis and Kurtic, Eldar and Alistarh, Dan},
  journal={arXiv preprint arXiv:2410.14649},
  year={2024}
}

@article{sakaguchi2021winogrande,
  title={Winogrande: An adversarial winograd schema challenge at scale},
  author={Sakaguchi, Keisuke and Bras, Ronan Le and Bhagavatula, Chandra and Choi, Yejin},
  journal={Communications of the ACM},
  volume={64},
  number={9},
  pages={99--106},
  year={2021},
  publisher={ACM New York, NY, USA}
}

@article{touvron2023llama,
  title={Llama 2: Open foundation and fine-tuned chat models},
  author={Touvron, Hugo and Martin, Louis and Stone, Kevin and Albert, Peter and Almahairi, Amjad and Babaei, Yasmine and Bashlykov, Nikolay and Batra, Soumya and Bhargava, Prajjwal and Bhosale, Shruti and others},
  journal={arXiv preprint arXiv:2307.09288},
  year={2023}
}

@article{team2024gemma,
  title={Gemma 2: Improving open language models at a practical size},
  author={Team, Gemma and Riviere, Morgane and Pathak, Shreya and Sessa, Pier Giuseppe and Hardin, Cassidy and Bhupatiraju, Surya and Hussenot, L{\'e}onard and Mesnard, Thomas and Shahriari, Bobak and Ram{\'e}, Alexandre and others},
  journal={arXiv preprint arXiv:2408.00118},
  year={2024}
}

@article{mihaylov2018can,
  title={Can a Suit of Armor Conduct Electricity? A New Dataset for Open Book Question Answering},
  author={Mihaylov, Todor and Clark, Peter and Khot, Tushar and Sabharwal, Ashish},
  journal={EMNLP},
  year={2018}
}

@article{zellers2019hellaswag,
  title={HellaSwag: Can a Machine Really Finish Your Sentence?},
  author={Zellers, Rowan and Holtzman, Ari and Bisk, Yonatan and Farhadi, Ali and Choi, Yejin},
  journal={ACL},
  year={2019},
}

@article{clark2019boolq,
  title={BoolQ: Exploring the Surprising Difficulty of Natural Yes/No Questions},
  author={Clark, Christopher and Lee, Kenton and Chang, Ming-Wei and Kwiatkowski, Tom and Collins, Michael and Toutanova, Kristina},
  journal={NAACL},
  year={2019}
}

@article{clark2018think,
  title={Think you have solved question answering? try arc, the ai2 reasoning challenge},
  author={Clark, Peter and Cowhey, Isaac and Etzioni, Oren and Khot, Tushar and Sabharwal, Ashish and Schoenick, Carissa and Tafjord, Oyvind},
  journal={arXiv preprint arXiv:1803.05457},
  year={2018}
}

@article{raffel2020exploring,
  title={Exploring the limits of transfer learning with a unified text-to-text transformer},
  author={Raffel, Colin and Shazeer, Noam and Roberts, Adam and Lee, Katherine and Narang, Sharan and Matena, Michael and Zhou, Yanqi and Li, Wei and Liu, Peter J},
  journal={JMLR},
  year={2020}
}

@article{merity2017pointer,
  title={Pointer Sentinel Mixture Models},
  author={Merity, Stephen and Xiong, Caiming and Bradbury, James and Socher, Richard},
  journal={ICLR},
  year={2017}
}

@article{zhao2023pytorch,
  title={Pytorch fsdp: experiences on scaling fully sharded data parallel},
  author={Zhao, Yanli and Gu, Andrew and Varma, Rohan and Luo, Liang and Huang, Chien-Chin and Xu, Min and Wright, Less and Shojanazeri, Hamid and Ott, Myle and Shleifer, Sam and others},
  journal={arXiv preprint arXiv:2304.11277},
  year={2023}
}

@article{micikevicius2022fp8,
  title={Fp8 formats for deep learning},
  author={Micikevicius, Paulius and Stosic, Dusan and Burgess, Neil and Cornea, Marius and Dubey, Pradeep and Grisenthwaite, Richard and Ha, Sangwon and Heinecke, Alexander and Judd, Patrick and Kamalu, John and others},
  journal={arXiv preprint arXiv:2209.05433},
  year={2022}
}

@article{paszke2019pytorch,
  title={Pytorch: An imperative style, high-performance deep learning library},
  author={Paszke, Adam and Gross, Sam and Massa, Francisco and Lerer, Adam and Bradbury, James and Chanan, Gregory and Killeen, Trevor and Lin, Zeming and Gimelshein, Natalia and Antiga, Luca and others},
  journal={Advances in neural information processing systems},
  volume={32},
  year={2019}
}

@article{liu2025proxsparse,
  title={PROXSPARSE: REGULARIZED LEARNING OF SEMI-STRUCTURED SPARSITY MASKS FOR PRETRAINED LLMS},
  author={Liu, Hongyi and Saha, Rajarshi and Jia, Zhen and Park, Youngsuk and Huang, Jiaji and Sabach, Shoham and Wang, Yu-Xiang and Karypis, George},
  journal={ICML},
  year={2025}
}

@article{kaplan2020scaling,
  title={Scaling laws for neural language models},
  author={Kaplan, Jared and McCandlish, Sam and Henighan, Tom and Brown, Tom B and Chess, Benjamin and Child, Rewon and Gray, Scott and Radford, Alec and Wu, Jeffrey and Amodei, Dario},
  journal={arXiv preprint arXiv:2001.08361},
  year={2020}
}

@article{faiz2024llmcarbon,
  title={LLMCarbon: Modeling the End-to-End Carbon Footprint of Large Language Models},
  author={Faiz, Ahmad and Kaneda, Sotaro and Wang, Ruhan and Osi, Rita Chukwunyere and Sharma, Prateek and Chen, Fan and Jiang, Lei},
  journal={ICLR},
  year={2024}
}

@article{bommasani2021opportunities,
  title={On the opportunities and risks of foundation models},
  author={Bommasani, Rishi},
  journal={arXiv preprint arXiv:2108.07258},
  year={2021}
}

@article{hoefler2021sparsity,
  title={Sparsity in deep learning: Pruning and growth for efficient inference and training in neural networks},
  author={Hoefler, Torsten and Alistarh, Dan and Ben-Nun, Tal and Dryden, Nikoli and Peste, Alexandra},
  journal={Journal of Machine Learning Research},
  volume={22},
  number={241},
  pages={1--124},
  year={2021}
}

@article{martens2020new,
  title={New insights and perspectives on the natural gradient method},
  author={Martens, James},
  journal={Journal of Machine Learning Research},
  volume={21},
  number={146},
  pages={1--76},
  year={2020}
}

@Misc{accelerate,
  title =        {Accelerate: Training and inference at scale made simple, efficient and adaptable.},
  author =       {Sylvain Gugger and Lysandre Debut and Thomas Wolf and Philipp Schmid and Zachary Mueller and Sourab Mangrulkar and Marc Sun and Benjamin Bossan},
  howpublished = {\url{https://github.com/huggingface/accelerate}},
  year =         {2022}
}

@article{8bit2022dettmers,
  title={8-bit Optimizers via Block-wise Quantization},
  author={Dettmers, Tim and Lewis, Mike and Shleifer, Sam and Zettlemoyer, Luke},
  booktitle={ICLR},
  year={2022}
}

@inproceedings{jacob2018quantization,
  title={Quantization and training of neural networks for efficient integer-arithmetic-only inference},
  author={Jacob, Benoit and Kligys, Skirmantas and Chen, Bo and Zhu, Menglong and Tang, Matthew and Howard, Andrew and Adam, Hartwig and Kalenichenko, Dmitry},
  booktitle={Proceedings of the IEEE conference on computer vision and pattern recognition},
  pages={2704--2713},
  year={2018}
}

@incollection{gholami2022survey,
  title={A survey of quantization methods for efficient neural network inference},
  author={Gholami, Amir and Kim, Sehoon and Dong, Zhen and Yao, Zhewei and Mahoney, Michael W and Keutzer, Kurt},
  booktitle={Low-power computer vision},
  pages={291--326},
  year={2022},
  publisher={Chapman and Hall/CRC}
}

@software{torchao,
  title={TorchAO: PyTorch-Native Training-to-Serving Model Optimization},
  author={torchao},
  url={https://github.com/pytorch/ao},
  license={BSD-3-Clause},
  month={oct},
  year={2024}
}

@article{frantar2023sparsegpt,
  title={Sparsegpt: Massive language models can be accurately pruned in one-shot},
  author={Frantar, Elias and Alistarh, Dan},
  journal={ICML},
  year={2023}
}

@article{sunsimple,
  title={A Simple and Effective Pruning Approach for Large Language Models},
  author={Sun, Mingjie and Liu, Zhuang and Bair, Anna and Kolter, J Zico},
  journal={ICLR},
  year={2024}
}

@article{zhang2022opt,
  title={Opt: Open pre-trained transformer language models},
  author={Zhang, Susan and Roller, Stephen and Goyal, Naman and Artetxe, Mikel and Chen, Moya and Chen, Shuohui and Dewan, Christopher and Diab, Mona and Li, Xian and Lin, Xi Victoria and others},
  journal={arXiv preprint arXiv:2205.01068},
  year={2022}
}

@article{meng2024alps,
  title={ALPS: Improved Optimization for Highly Sparse One-Shot Pruning for Large Language Models},
  author={Meng, Xiang and Behdin, Kayhan and Wang, Haoyue and Mazumder, Rahul},
  journal={NeurIPS},
  year={2024},
}

@article{zeng2018mlprune,
  title={MLPrune: Multi-layer pruning for automated neural network compression},
  author={Zeng, Wenyuan and Urtasun, Raquel},
  journal={arXiv},
  year={2018}
}

@article{shin2024rethinking,
  title={Rethinking Pruning Large Language Models: Benefits and Pitfalls of Reconstruction Error Minimization},
  author={Shin, Sungbin and Park, Wonpyo and Lee, Jaeho and Lee, Namhoon},
  journal={EMNLP},
  year={2024}
}

@article{fang2024maskllm,
  title={MaskLLM: Learnable Semi-Structured Sparsity for Large Language Models},
  author={Fang, Gongfan and Yin, Hongxu and Muralidharan, Saurav and Heinrich, Greg and Pool, Jeff and Kautz, Jan and Molchanov, Pavlo and Wang, Xinchao},
  journal={NeurIPS},
  year={2024}
}

@article{bai2024sparsellm,
  title={SparseLLM: Towards Global Pruning of Pre-trained Language Models},
  author={Bai, Guangji and Li, Yijiang and Ling, Chen and Kim, Kibaek and Zhao, Liang},
  journal={NeurIPS},
  year={2024}
}

@article{hu2022lora,
  title={Lora: Low-rank adaptation of large language models.},
  author={Hu, Edward J and Shen, Yelong and Wallis, Phillip and Allen-Zhu, Zeyuan and Li, Yuanzhi and Wang, Shean and Wang, Lu and Chen, Weizhu and others},
  journal={ICLR},
  year={2022}
}

@article{han2015learning,
  title={Learning both weights and connections for efficient neural network},
  author={Han, Song and Pool, Jeff and Tran, John and Dally, William},
  journal={NeurIPS},
  year={2015}
}

@article{
yin2024outlier,
title={Outlier Weighed Layerwise Sparsity ({OWL}): A Missing Secret Sauce for Pruning {LLM}s to High Sparsity},
author={Lu Yin and You Wu and Zhenyu Zhang and Cheng-Yu Hsieh and Yaqing Wang and Yiling Jia and Gen Li and AJAY KUMAR JAISWAL and Mykola Pechenizkiy and Yi Liang and Michael Bendersky and Zhangyang Wang and Shiwei Liu},
journal={ICML},
year={2024},
}

@article{
lee2025safe,
title={{SAFE}: Finding Sparse and Flat Minima to Improve Pruning},
author={Dongyeop Lee and Kwanhee Lee and Jinseok Chung and Namhoon Lee},
journal={ICML},
year={2025},
}

@article{
zhang2024dynamic,
title={Dynamic Sparse No Training:  Training-Free Fine-tuning for Sparse {LLM}s},
author={Yuxin Zhang and Lirui Zhao and Mingbao Lin and Sun Yunyun and Yiwu Yao and Xingjia Han and Jared Tanner and Shiwei Liu and Rongrong Ji},
journal={ICLR},
year={2024},
}

@article{
boza2024fast,
title={Fast and Effective Weight Update for Pruned Large Language Models},
author={Vladim{\'\i}r Bo{\v{z}}a},
journal={Transactions on Machine Learning Research},
year={2024},
}

@article{natarajan1995sparse,
  title={Sparse approximate solutions to linear systems},
  author={Natarajan, B.K.},
  journal={SIAM Journal on Computing},
  volume={24},
  number={2},
  pages={227--234},
  year={1995},
  publisher={SIAM}
}

@article{boyd2011distributed,
  title={Distributed optimization and statistical learning via the alternating direction method of multipliers},
  author={Boyd, Stephen and Parikh, Neal and Chu, Eric and Peleato, Borja and Eckstein, Jonathan and others},
  journal={Foundations and Trends{\textregistered} in Machine learning},
  year={2011},
}

@article{
huang2025determining,
title={Determining Layer-wise Sparsity for Large Language Models Through a Theoretical Perspective},
author={Weizhong Huang and Yuxin Zhang and Xiawu Zheng and Fei Chao and Rongrong Ji},
journal={ICML},
year={2025}
}

@article{grattafiori2024llama3herdmodels,
  title={The Llama 3 Herd of Models},
  author={Grattafiori, Aaron and Dubey, Abhimanyu and Jauhri, Abhinav and Pandey, Abhinav and Kadian, Abhishek and Al-Dahle, Ahmad and Letman, Aiesha and Mathur, Akhil and Schelten, Alan and Vaughan, Alex and others},
  journal={arXiv},
  year={2024}
}

@article{lee2019snip,
  title={Snip: Single-shot network pruning based on connection sensitivity},
  author={Lee, Namhoon and Ajanthan, Thalaiyasingam and Torr, Philip HS},
  journal={ICLR},
  year={2019}
}

@article{frankle2018lottery,
  title={The lottery ticket hypothesis: Finding sparse, trainable neural networks},
  author={Frankle, Jonathan and Carbin, Michael},
  journal={ICLR},
  year={2019}
}

@article{huang21ADMMQ,
  title = { Alternating Direction Method of Multipliers for Quantization },
  author = {Huang, Tianjian and Singhania, Prajwal and Sanjabi, Maziar and Mitra, Pabitra and Razaviyayn, Meisam},
  journal = {AISTATS},
  year = {2021}
}

@article{macko2025macko,
  title={MACKO: Sparse Matrix-Vector Multiplication for Low Sparsity},
  author={Macko, Vladim{\'\i}r and Bo{\v{z}}a, Vladim{\'\i}r},
  journal={arXiv preprint arXiv:2511.13061},
  year={2025}
}
\bibliographystyle{iclr2026_conference}

\newpage

\appendix

\section{Proof of \cref{thm:gpaq_conv}}
\label{app: proof_qgpa}
Here we present the convergence proof of \oursq{}.
Formally, we prove the convergence of the following algorithm:

\begin{algorithm}[]
	\caption{\oursq{}}
	\label{alg: gpaq}
	\begin{algorithmic}[1]
	    \State {\textbf{Input}}: Constant~$\lambda>0$; initial points $x_0$, $ u_0  \in \mathbb{R}^d$
		\For{$r = 0,1,2,\ldots$}
		\State \textbf{Update $y$}: \quad  $\Pi_{\mathcal{S}} (x^t+ \lambda^{-1} u^t)$ \label{step:ADMMQYupdate}%\arg\min_{y\in\mathcal{S}} \frac{1}{2}\left\|y-x^t-\lambda^{-1} u^t\right\|^2=
		% \State \textbf{Update $y$}: \quad  $\quant [\Pi_{\mathcal{S}} (x^t+ \lambda^{-1} u^t)]$ \label{step:ADMMQYupdate}%\arg\min_{y\in\mathcal{S}} \frac{1}{2}\left\|y-x^t-\lambda^{-1} u^t\right\|^2=
		\State \textbf{Update $x$} by finding a point $x^{t+1}$ satisfying $\nabla f(x^{t+1})+\mathcal{Q}[ u^t+\lambda(x^{t+1}-z^{t+1})]$=0 and $\|x^{t+1}-x^{t+1}_\star\|\leq \gamma \min\;\{\|x^{t+1}-z^{t+1}\|,\|x^{t+1}-x^t\| \}$
        \label{xupdate}
		\State \textbf{Update $ u$}:\quad  $ u^{t+1}=\quant [ u^t+\lambda(x^{t+1}-z^{t+1})]$
		\EndFor
		%\State{\textbf{Output}: $y$ {\color{red}Tianjian: should we use $x$ or $\bar{x}?$}}
	\end{algorithmic}
\end{algorithm}

First, let us define:
\begin{align}
    e^t & = \nabla_x\mathcal{L}(x^t,z^t, u^{t-1}) \\
    & = \nabla f(x^t) +  u^{t-1} + \lambda(x^t-z^t) \\
    & =  u^{t-1} + \lambda(x^t-z^t) - \mathcal{Q}[ u^{t-1} + \lambda(x^t-z^t)]. %\\
    % & =  u^{r+1/2} - \mathcal{Q}[ u^{r+1/2}]
\end{align}

Thus, we can express the $ u$ step in terms of $e^t$ as follows 
\begin{align}
     u^{t+1} & = \mathcal{Q}[ u^t + \lambda(x^{t+1}-z^{t+1})] \\
    & =  u^t + \lambda(x^{t+1}-z^{t+1}) - e^{t+1}
\end{align}

\begin{lemma} 
\label{lemma: strong-convexity}
Due to $(\lambda - \mu)$-strong convexity and $(\beta+\lambda)$-smoothness of $\mathcal{L}(\cdot, z^t,  u^{t-1})$, we know that
\begin{align}
    (\lambda - \mu) \|x^t-x_\star^t\|\leq \|e^t\| \leq (\lambda + \beta)\|x^t-x_\star^t\|
\end{align}
Moreover, due to strong convexity we also know that:
\begin{align}
    \langle e^t, x^t-x_\star^t\rangle \geq (\lambda - \mu) \|x^t-x_\star^t\|^2
\end{align}
\end{lemma}

\begin{lemma} 
\label{lemma: lower_bound_inexact}
If $\lambda\geq \beta$ and we also assume that the iterates $x^t$ stay bounded. Then there exists a non-negative number $\bar{D}$ s.t. $\|x^t-z^t\|\leq \bar{D}$. With this definition, 
\begin{align}
    \mathcal{L}(x^t, z^t,  u^t) \geq f_{\min} - {\gamma}(\lambda+\beta)\bar{D}^2
\end{align}
\end{lemma}
\begin{proof}
    Note that
    \begin{align}
      \mathcal{L}(x^t, z^t,  u^t) &=
      f(x^t) + \langle  u^t, x^t-z^t\rangle + \frac{\lambda}{2}\|x^t-z^t\|^2\\
      & = \underbrace{f(x^t) + \langle\nabla f(x^t), z^t-x^t \rangle + \frac{\lambda}{2}\|x^t-z^t\|^2}_{\geq f(z^t)}+ \langle e^t, x^t-z^t\rangle\\
      &\geq f(z^t) - \|e^t\|\|x^t-z^t\|\\
      &\geq f_{\min} - {\gamma} (\lambda+\beta)\bar{D}^2
    \end{align}
    where the last inequality is due to the assumptions and Lemma \ref{lemma: strong-convexity}.
\end{proof}
Now let us prove sufficient decrease on $\mathcal{L}$ in each iteration.
\begin{lemma}
\label{lemma: decrease_inexact}
Let the assumptions of Lemma \ref{lemma: lower_bound_inexact} be true. Also, assume that the parameters $\lambda$ and ${\gamma}$ are chosen such that
\begin{equation}
     \frac{\beta^2}{\lambda} + \frac{\beta(\lambda + \beta) \gamma}{\lambda}  + \frac{{\gamma}^2(\lambda+\beta)}{2}-\frac{(1-{\gamma})^2(\lambda - \mu)}{2} < 0.
\end{equation}
Note that $\lambda-\mu\geq 0$. Then, we have
\begin{equation}
    \lim_{r\rightarrow\infty}\|x^{t+1}-x^t\| = 0.
\end{equation}

\end{lemma}

\begin{proof}
Let 
\begin{equation}
\nonumber
\begin{split}
\mathcal{L}(x^{t+1},z^{t+1}, u^{t+1})-\mathcal{L}(x^t,z^t, u^t)=\underbrace{\mathcal{L}(x^{t+1},z^{t+1}, u^{t+1})-\mathcal{L}(x^{t+1},z^{t+1}, u^t)}_{(A)}\\+ \underbrace{\mathcal{L}(x^{t+1},z^{t+1}, u^t)-\mathcal{L}(x^t,z^t, u^t)}_{(B)}.
\end{split}
\end{equation}
\\
We want to show that $(A)+(B)\leq0$.

\begin{equation}
\nonumber
    (A)=\langle u^{t+1},x^{t+1}-z^{t+1}\rangle-\langle u^t,x^{t+1}-z^{t+1}\rangle=\lambda^{-1} \bigg( \left\| u^{t+1}- u^t\right\|^2+\langle e^{t+1},   u^{t+1}- u^t \rangle \bigg).
\end{equation}
Using our definitions, we have
\begin{align}
    (A) 
    &= \lambda^{-1} \bigg( \| u^{t+1}- u^t\|^2 +\langle e^{t+1},   u^{t+1}- u^t \rangle \bigg) \\
    &= \lambda^{-1} \bigg( \| \nabla f(x^{t+1})-\nabla f(x^t) \|^2 +\langle e^{t+1},  \nabla f(x^{t+1})-\nabla f(x^t) \rangle \bigg) \\
    & \leq \lambda^{-1} \bigg( \| \nabla f(x^{t+1})-\nabla f(x^t) \|^2 + \| e^{t+1} \|  \| \nabla f(x^{t+1})-\nabla f(x^t) \|  \bigg) \\
    & \leq \lambda^{-1} \bigg( \beta^2 \| x^{t+1} - x^t \|^2 + \beta \| e^{t+1} \|\| x^{t+1} - x^t \| \bigg) \\
    & \leq \lambda^{-1} \bigg( \beta^2 \| x^{t+1} - x^t \|^2 + \beta (\lambda + \beta) \| x^{t+1} - x^{t+1}_\star \|\| x^{t+1} - x^t \| \bigg) \\
    & \leq \lambda^{-1} \bigg( \beta^2 \| x^{t+1} - x^t \|^2 + \beta (\lambda + \beta) \gamma \| x^{t+1} - x^t \|^2 \bigg) \\
    & = \lambda^{-1} \beta \bigg( \beta  +  (\lambda + \beta) \gamma  \bigg) \| x^{t+1} - x^t \|^2,
\end{align}
where the last inequality is due to Lemma \ref{lemma: strong-convexity} and the way $x^t$ is chosen in Algorithm \ref{alg: gpaq}.

On the other hand:
\begin{equation}
\nonumber
    \begin{aligned}
        (B)&=\mathcal{L}(x^{t+1},z^{t+1}, u^t)-\mathcal{L}(x^t,z^t, u^t)\\
        &=\mathcal{L}(x^{t+1},z^{t+1}, u^t)-\mathcal{L}(x^t,z^{t+1}, u^t)+\underbrace{\mathcal{L}(x^t,z^{t+1}, u^t)-\mathcal{L}(x^t,z^t, u^t)}_{\leq 0 \text{ (due to update of $y$)}}\\
        &\leq \mathcal{L}(x^{t+1},z^{t+1}, u^t)-\mathcal{L}(x^t,z^{t+1}, u^t)\\
        & = \underbrace{\mathcal{L}(x^{t+1},z^{t+1}, u^t) - \mathcal{L}(x_\star^{t+1},z^{t+1}, u^t)}_{\leq\frac{\beta+\lambda}{2}\|x^{t+1}-x_\star^{t+1}\|^2 } + \underbrace{\mathcal{L}(x_\star^{t+1},z^{t+1}, u^t)- \mathcal{L}(x^t,z^{t+1}, u^t)}_{\leq -\frac{(\lambda - \mu)}{2}\|x_\star^{t+1}-x^t\|^2}\\
        &\leq \frac{\beta+\lambda}{2}\|x^{t+1}-x_\star^{t+1}\|^2 -\frac{(\lambda - \mu)}{2}\|x_\star^{t+1}-x^t\|^2 ,
    \end{aligned}
\end{equation}
Now note that $\|x^t-x_\star^{t+1}\|\geq (1-{\gamma})\|x^{t+1}-x^t\|$ and $\|x^{t+1}-x_\star^{t+1}\|\leq {\gamma} \|x^{t+1}-x^t\|$ because of the update rules of Algorithm \ref{alg: gpaq}. Plugging in these, we get
\begin{align}
    (B) \leq \bigg(\frac{{\gamma}^2 (\lambda+\beta)}{2}-\frac{(1-{\gamma})^2(\lambda - \mu)}{2}\bigg)\|x^{t+1}-x^t\|^2
\end{align}
Now combining the inequalities for $(A)$ and $(B)$, we have
\begin{align}
    &\mathcal{L}(x^{t+1},z^{t+1}, u^{t+1})-\mathcal{L}(x^t,z^t, u^t) \\
    &\leq \underbrace{\bigg( \frac{\beta^2}{\lambda} + \frac{\beta(\lambda + \beta) \gamma}{\lambda}  + \frac{{\gamma}^2(\lambda+\beta)}{2}-\frac{(1-{\gamma})^2(\lambda - \mu)}{2}\bigg)}_{\alpha}\|x^{t+1}-x^t\|^2 
\end{align}
Now for any $T$:
\begin{align}
    f_{\min} - {\gamma}(\lambda+\beta)\bar{D}^2 
    &\leq\mathcal{L}(x^{T+1}, z^{T+1},  u^{T+1})\\
    &= \mathcal{L}(x^0, z^0,  u^0)+ \sum_{t=0}^T \mathcal{L}(x^{t+1},z^{t+1}, u^{t+1})-\mathcal{L}(x^t,z^t, u^t)\\
    &\leq \alpha\sum_{t=0}^{T}\|x^{t+1}-x^t\|^2 + \mathcal{L}(x^0, z^0,  u^0).
\end{align}
Now if the parameters are chosen appropriately such that $\alpha<0$, then the right hand side of the above inequality is decreasing as $T$ increases, while the left hand side is constant. Therefore, we have $\lim_{T\rightarrow\infty}\sum_{t=0}^T\|x^{t+1}-x^t\|^2<\infty$. Thus, $\lim_{r\rightarrow\infty}\|x^{t+1}-x^t\| = 0$.
\end{proof}

\begin{theorem}
Assume that all the assumptions of Lemma \ref{lemma: decrease_inexact} is satisfied. Then, For any limit point $(\bar{x}, \bar{z}, \bar{\lambda})$ of the Algorithm \ref{alg: gpaq}, $\bar{x}$ is a $\lambda$-stationary solution of the problem.
\end{theorem}
\begin{proof}

Consider a sub-sequence $( x^{r_t}, z^{r_t}, u^{r_t})$, for $t=0,\cdots$ which converges to $(\bar{x}, \bar{z},\bar{u})$. First of all due to Lemma \ref{lemma: decrease_inexact}, we know that $\lim_{t\rightarrow \infty}\|x^{r_t+1}- x^{r_t}\|=0$ and $\lim_{t\rightarrow\infty}\|x^{r_t-1}- x^{r_t}\|=0$. Thus, 
\begin{equation}
    \lim_{t\rightarrow \infty} x^{r_t+1} = \bar{x} ~~\&~~\lim_{t\rightarrow \infty} x^{r_t-1} = \bar{x} 
\end{equation}
Moreover, due to the updates of the algorithm
\begin{equation}
    \lim_{t\rightarrow\infty} \|x^{r_t+1}-x_\star^{r_t+1}\| \leq \lim_{t\rightarrow\infty} {\gamma} \|x^{r_t+1}-x^{r_t}\| = 0 ~~\&~~\lim_{t\rightarrow\infty} \|x^{r_t}-x_\star^{r_t}\| \leq \lim_{t\rightarrow\infty} {\gamma} \|x^{r_t}-x^{r_t-1}\| = 0
\end{equation}
Thus, $\lim_{t\rightarrow\infty} e^{r_t} = \lim_{t\rightarrow\infty} e^{r_t+1} = 0$, which means
\begin{align}
    \bar{u} = \lim_{t\rightarrow \infty}  u^{r_t} = - \lim_{t\rightarrow \infty} (\nabla f( x^{r_t})-e^{r_t}) = -\nabla f(\bar{x})\\
    \lim_{t\rightarrow \infty}  u^{r_t+1} =  -\lim_{t\rightarrow \infty} (\nabla f(x^{r_t+1})-e^{r_{t+1}}) = -\nabla f(\bar{x}) 
\end{align}
Thus, $\lim_{t\rightarrow \infty}  u^{r_t+1} = \bar{u}$. 

Also, as $\mathcal{S}$ is finite, there exists a large enough T, such that $ z^{r_t} = \bar{y}$ for $t\geq T$. Again due to the fact that $\mathcal{S}$ is finite, we can re-fine the sub-sequence such that $z^{r_t+1} = \hat{y}$. Thus, without loss of generality assume that these two conditions hold, i.e. $ z^{r_t} = \bar{y}$ and $z^{r_t+1} = \hat{y}$ for all $t$ for an appropriately refined sub-sequence. This means that
\begin{equation}
\label{eq: y_hat}
    \hat{y} \in \arg\min_{x}\|x-( x^{r_t}+\lambda^{-1} u^{r_t})\|
\end{equation}
Moreover, $ u^{r_t+1} =  u^{r_t} + \lambda(x^{r_t+1}-\hat{y})$. Taking the $\lim_{t\rightarrow\infty}$ from both sides, we get
\begin{equation}
    \hat{y} = \bar{x}.
\end{equation}
Combining the above with \Eqref{eq: y_hat} we can easily see that
\begin{equation}
    \|\bar{x}-( x^{r_t}+\lambda^{-1} u^{r_t})\| \leq \|a_i-( x^{r_t}+\lambda^{-1} u^{r_t})\|,~i=0,\cdots,N
\end{equation}
Taking the limits $\lim_{t\rightarrow\infty}$ from both hand sides of the inequality for all the points $a_i$ we have
\begin{equation}
    \|\bar{x}-(\bar{x}+\lambda^{-1}\bar{u})\| \leq \|a_i-(\bar{x}+\lambda^{-1}\bar{u})\|,~i=0,\cdots,N.
\end{equation}
Thus, 
\begin{equation}
    \bar{x}\in \arg\min_{x\in\mathcal{S}}\|x-(\bar{x}-\lambda^{-1}\nabla f(\bar{x}))\|,
\end{equation}
where we used the fact that $\bar{u} = -\nabla f(\bar{x})$.
\end{proof}

\begin{table}[h]
\centering
\caption{Global hyperparameters of \ours shared across all models.}
\label{tab:global-hyperparams}
\resizebox{0.3\linewidth}{!}{
\begin{tabular}{lc}
\toprule
\textbf{Hyperparameter} & \textbf{Value} \\
\midrule
LR schedule          & Linear decay \\
% Penalty schedule     & Cosine \\
Interval $k$           & 32 \\
Adam $(\beta_1,\beta_2)$             & (0.9, 0.999) \\
Total \# datapoints  & 32768 \\
% Batch size             & 8 \\
% Training steps         & 4096 \\
\bottomrule
\end{tabular}}
\end{table}

\begin{table}[t!]
    \centering
    \caption{Learning rate ($\eta$), penalty ($\lambda$), and penalty schedule configuration across models at different sparsity levels.}
    \label{tab:hyperparams}
    \scriptsize
    \setlength{\tabcolsep}{4pt}
    \renewcommand{\arraystretch}{1.1}
    \begin{tabular}{ccc|cccccc}
        \toprule
         & \textbf{Sparsity} & \textbf{$\lambda$ sched.} &
        OPT-125M & OPT-1.3B & Gemma-2-2B & LLaMA-3.2-3B & LLaMA-2-7B & LLaMA-2-13B \\
        \midrule
        \multirow{5}{*}{$\eta$}
        & 50\% & --
            & \cellcolor{gray!28} 1e-4
            & \cellcolor{gray!28}
            & \cellcolor{gray!5}
            & \cellcolor{gray!8} 5e-5
            & \cellcolor{gray!8} 5e-5
            & \cellcolor{gray!8} \\
        & 60\% & --
            & \cellcolor{gray!38} 2e-4
            & \cellcolor{gray!28} \multirow{-2}{*}{1e-4}
            & \cellcolor{gray!5}
            & \cellcolor{gray!28} 1e-4
            & \cellcolor{gray!28} 1e-4
            & \cellcolor{gray!8} \\
        & 70\% & --
            & \cellcolor{gray!28} 1e-4
            & \cellcolor{gray!8} 5e-5
            & \cellcolor{gray!5}
            & \cellcolor{gray!8} 5e-5
            & \cellcolor{gray!8}
            & \cellcolor{gray!8} \\
        & 80\% & --
            & \cellcolor{gray!38}
            & \cellcolor{gray!28}
            & \cellcolor{gray!5} \multirow{-4}{*}{2e-5}
            & \cellcolor{gray!28}
            & \cellcolor{gray!8} \multirow{-2}{*}{5e-5}
            & \cellcolor{gray!8} \\
        & 90\% & --
            & \cellcolor{gray!38} \multirow{-2}{*}{2e-4}
            & \cellcolor{gray!28} \multirow{-2}{*}{1e-4}
            & \cellcolor{gray!8} 5e-5
            & \cellcolor{gray!28} \multirow{-2}{*}{1e-4}
            & \cellcolor{gray!28} 1e-4
            & \cellcolor{gray!8} \multirow{-5}{*}{5e-5} \\
        \midrule
        \multirow{5}{*}{$\lambda$}
        & 50\% & constant
            & \cellcolor{gray!8} 1e-4
            & \cellcolor{gray!18}
            & \cellcolor{gray!50} 2e-1
            & \cellcolor{gray!18} 1e-3
            & \cellcolor{gray!18} 1e-3
            & \cellcolor{gray!22} 2e-3 \\
        & 60\% & constant
            & \cellcolor{gray!18} 1e-3
            & \cellcolor{gray!18} \multirow{-2}{*}{1e-3}
            & \cellcolor{gray!30}
            & \cellcolor{gray!22} 2e-3
            & \cellcolor{gray!30} 5e-3
            & \cellcolor{gray!18} 1e-3  \\
        & 70\% & cosine
            & \cellcolor{gray!22} 2e-3
            & \cellcolor{gray!30} 5e-3
            & \cellcolor{gray!34} 1e-2
            & \cellcolor{gray!30} 5e-3
            & \cellcolor{gray!38}
            & \cellcolor{gray!38} 2e-2 \\
        & 80\% & cosine
            & \cellcolor{gray!18}
            & \cellcolor{gray!18}
            & \cellcolor{gray!30} 5e-3
            & \cellcolor{gray!18}
            & \cellcolor{gray!38} \multirow{-2}{*}{2e-2}
            & \cellcolor{gray!45} 5e-2 \\
        & 90\% & cosine
            & \cellcolor{gray!18} \multirow{-2}{*}{1e-3}
            & \cellcolor{gray!18} \multirow{-2}{*}{1e-3}
            & \cellcolor{gray!12} 5e-4
            & \cellcolor{gray!18} \multirow{-2}{*}{1e-3}
            & \cellcolor{gray!18} 1e-3
            & \cellcolor{gray!22} 2e-3 \\
        \bottomrule
    \end{tabular}
\end{table}

\begin{table}[t!]
    \centering
    \caption{Batch size (BS) and number of steps used at different sparsity levels.}
    \label{tab:bs_steps_rules}
    \scriptsize
    \setlength{\tabcolsep}{6pt}
    \renewcommand{\arraystretch}{1.15}
    \begin{tabular}{c|cc|cc}
        \toprule
        \textbf{Model(s)} &
        \multicolumn{2}{c|}{\textbf{50--60\%}} &
        \multicolumn{2}{c}{\textbf{70--90\%}} \\
        & \textbf{BS} & \textbf{Steps} & \textbf{BS} & \textbf{Steps} \\
        \midrule
        OPT-125M &
        16 & 2048 &
        \multicolumn{2}{c}{\multirow{2}{*}{8 \quad  \quad 4096}} \\
        OPT-1.3B, Gemma-2-2B, LLaMA-3.2-3B, LLaMA-2-7B, LLaMA-2-13B &
        32 & 1024 &
        \multicolumn{2}{c}{} \\
        \bottomrule
    \end{tabular}
\end{table}

\section{Experimental details}
\label{Appendix: Experimental details}
\subsection{Implementation and reproduction details}
\label{Appendix: Implementation and reproduction details}
Our implementation is based on PyTorch \citep{paszke2019pytorch}, 
using the HuggingFace \texttt{transformers} and \texttt{datasets} libraries for model and data loading. 
\ours is implemented over HuggingFace \texttt{Trainer}, supporting distributed training via PyTorch FSDP-2 \citep{zhao2023pytorch} with HuggingFace \texttt{Accelerate}.

All experimental results in this work are obtained with unified codebase, 
while baseline methods are reproduced using their original implementations whenever available.  
The environment configuration (dependencies, versions, and training scripts)
can be found at \texttt{https://github.com/log-postech/elsa}.
% will be released together with the code to ensure full reproducibility.  

Experiments are conducted on NVIDIA A100/H200 GPUs, 
with the number of GPUs scaled to model size: 
2$\times$GPUs for 1.3B–3B models, 4$\times$A100 GPUs for 7B models, 
and 4$\times$H200 GPUs for 13B and 27B models.  

\subsection{Details for \cref{subsec:main_results}}
\label{Appendix:details_5.1}
\paragraph{Calibration/Training data.}
To obtain baseline results (Wanda, SparseGPT, ALPS, L-ADMM, SAFE, SparseLLM), 
we follow the convention of \citet{frantar2023sparsegpt}, sampling 128 calibration sequences 
from the C4 dataset with sequence length 2048.  
For \ours, we adopt the same strategy, but use larger calibration sets 
to account for the iterative nature of our optimization.  

\paragraph{Training details.}
We train \ours with 32{,}768 data points, where each data point has sequence length of 2048. Batch size and number of steps differ across model, which is presented at table \cref{tab:bs_steps_rules}. We use Adam as the base optimizer.  
The penalty parameter is kept constant for moderate sparsity levels (50-60\%), while we use cosine schedule, which gradually increases the penalty parameter from $0$ at the start to $\lambda$ at the end of training.
All model parameters and optimizer states uses full precision for training (except for memory-efficient setting and ablations), and automatic mixed precision with bf16 precision is used for efficient training.
A full list of hyperparameter configurations is provided in \cref{tab:global-hyperparams,tab:hyperparams}.  

\paragraph{Evaluation.}
Perplexity is measured on the held-out (validation) C4 \citep{raffel2020exploring} and Wikitext2 \citep{merity2017pointer} datasets.  
Zero-shot performance is evaluated with \texttt{lm-eval-harness} across seven standard tasks: ARC-Easy/Challenge (ARC-E/C) \citep{clark2018think}, BoolQ \citep{clark2019boolq}, HellaSwag \citep{zellers2019hellaswag}, OpenBookQA (OBQA) \citep{mihaylov2018can}, RTE \citep{zeng2018mlprune}, and Winogrande \citep{sakaguchi2021winogrande}, 
and we report the average accuracy as in \cref{subsec:main_results}.

\subsection{Details for \cref{subsec:extreme_sparsity}}
\label{Appendix:5.2}
For the baseline (Wanda+LoRA / Full), we first prune the pretrained model with Wanda, and then retrain the remaining parameters using either LoRA \citep{hu2022lora} or full fine-tuning (Full).
For a fair comparison, we use the same number data points with same sequence length as \ours, with hyperparameters tuned separately for each method.

\subsection{Details for \cref{subsec:scaling_results}}
\label{Appendix:details_5.2}
We ran \oursq on Gemma-2-27B using 4$\times$H200 GPUs.  
Fp8 representations for ADMM states $(u,z)$ were implemented based on the \texttt{torchao} framework\citep{torchao}, where we further extended the implementation to fully support \texttt{DTensor}, as required by the FSDP-2 framework for distributed training.
For this setting, we used a learning rate of $\eta=2\times10^{-5}$ and penalty parameter $\lambda=0.002$, using cosine penalty scheduling.
\subsection{Details for \cref{subsec:other_sparsity}}
For N:M semi-structured sparsity, we use the same hyperparameter configuration as for 50\% unstructured sparsity.  

For non-uniform sparsity comparisons, we evaluate \ours on LLaMA-3-8B using the hyperparameters of LLaMA-2-7B at 70\% sparsity, 
while the results of SparseGPT, OWL, and EvoPress are taken directly from \citet{sieberling2024evopress}.  
For \ours (EvoPress), we adopt the non-uniform sparsity configurations provided in the official EvoPress repository, 
and initialize \ours with these sparsity budgets while keeping the same training hyperparameters.

\subsection{Details for \cref{subsec:ablation}}
\label{Appendix:5.4}
For objective ablation, we used the OPT-125M model at 90\% sparsity, fixing the total number of optimization steps to 4,096 and varying the data count from 256 up to 32,684, using the same hyperparameter configurations as in \cref{tab:hyperparams}.

\section{Additional results}
\label{Appendix:Additional results}
Here we provide additional results for different sparsity patterns (\Cref{subsec:other_sparsity}), ablation studies on objective function and projection step (\Cref{subsec:ablation}), and a numerical results used to make visual plots in the main text followed by additional result reporting LLaMA-2-13B zero-shot task accuracy.
%% tables
\newpage 

\begin{wraptable}{r}{0.35\textwidth} 
    \centering
    \vspace{-1.3em}
    \caption{Perplexity of LLaMA-3-8B at 70\% sparsity. 
    \ours outperforms prior allocation methods.}
    \label{tab:llama8b_evopress70}
    \resizebox{\linewidth}{!}{%
    \scriptsize
    \begin{tabular}{lcc}
        \toprule
        \textbf{Method} & \texttt{Wiki}$(\downarrow)$ & \texttt{C4}$(\downarrow)$ \\
        \midrule
        % Dense & 5.54 & 7.10  \\
        % \midrule
        SparseGPT & 85.84 & 98.35  \\
        OWL & 48.07 & 52.32  \\
        EvoPress & 28.76 & 33.72 \\
        \rowcolor{gray!10}
        \ours (EvoPress) & 26.11 & 29.33\\
        \rowcolor{gray!10}
        \ours &\textbf{24.97} & \textbf{29.09} \\
        \bottomrule
    \end{tabular}
    }
\end{wraptable}
\subsection{Other sparsity patterns}
\label{subsec:other_sparsity}
In this section, we analyze whether \ours can adapt to other sparsity patterns including (i) N:M semi-structured sparsity and (ii) non-uniform sparsity over different layers.

\begin{table}[t!]
    \centering
    \scriptsize
    \setlength{\tabcolsep}{4pt}
    \vspace{1em}
    \caption{
        Perplexity and zero-shot prediction accuracy of LLaMA-2-7B under N:M semi-structured sparsity.
        \ours compares competitively to other methods, demonstrating its adaptivity.
        Note that 2:4 and 4:8 patterns are only 50\% sparsity levels.
    }
    \label{tab:llama7b-semi}
    \resizebox{\linewidth}{!}{%
    \begin{tabular}{clcc|*{8}{>{\centering\arraybackslash}p{0.9cm}}}
        \toprule
        & & \multicolumn{2}{c}{\textbf{Perplexity} ($\downarrow$)} & \multicolumn{8}{c}{\textbf{Tasks} ($\uparrow$)} \\
        \cmidrule(lr){3-4} \cmidrule(lr){5-12}
        \textbf{Sparsity} & \textbf{Method} & \texttt{Wiki} & \texttt{C4} & ARC-C & ARC-E & BoolQ & HellaSwag & OBQA & RTE & Winogrande & Avg. \\
        \midrule
        0\% & Dense & 5.47 & 7.26 & 43.35& 76.26& 77.68& 57.14& 31.40& 62.82& 69.06 & 59.67  \\
        \midrule
        \multirow{8}{*}{2:4}
            & Magnitude & 37.76 & 74.66  & 30.12 & 61.87 & 59.85 & 45.45 & 21.80 & 52.35 & 61.01 & 47.49 \\
            & Wanda     & 12.13 & 15.63  & 30.46 & 61.83 & 68.26 & 41.28 & 24.20 & 53.07 & 62.51 & 48.80 \\
            & SparseGPT & 10.87 & 13.61 & 30.97 & 64.06 & 67.61 & 43.47 & 24.20 & \underline{56.32} & \underline{66.38} & 50.43
  \\
            & L-ADMM    & 10.19 & 12.51 & \underline{32.85} & \underline{66.04} & 68.81 & 45.05 & \underline{25.40} & \underline{56.32} & \underline{66.38} & \underline{51.55} \\
            & ALPS      &\underline{9.945}& \textbf{12.09} & \textbf{34.47} & \textbf{68.86} & \textbf{73.79} & \textbf{49.40} & \textbf{27.60} & 55.60 & \textbf{67.25} & \textbf{53.85} \\
            & SAFE & \textbf{9.914}& \underline{12.53} & 30.46 & 63.43 & 66.42 & 44.66 & 21.60 & 53.07 & 61.80 & 48.78 \\
            & SparseLLM & 11.29 & 13.95 & 30.55 & 61.91 & \underline{71.10} & 43.62 & 24.40 & \textbf{57.40} & 65.82 & 50.69 \\
            \rowcolor{gray!10}\cellcolor{white}
            & \ours     & 10.15 & 12.34 & 31.49& 61.24 &	66.36&	\underline{47.87} &	23.60&	52.71&	63.85 & 49.59  \\
        \midrule
        \multirow{7}{*}{4:8}
            & Magnitude & 15.91 & 31.60 & \textbf{36.01} & 64.81 & 63.09 & \textbf{50.05} & 26.00 & 52.35 & 62.19 & 50.64 \\
            & Wanda     & 8.603 & 11.33  & 34.47 & 67.05 & \underline{72.87} & 46.98 & 26.80 & 54.15 & 66.93 & 52.75 \\
            & SparseGPT & 8.508 & 10.81  & 34.81 & \textbf{68.56} & 71.77 & 48.26 & \underline{27.80} & \underline{56.68} & \textbf{68.11} & \underline{53.71}\\
            & L-ADMM    & 8.12 & \underline{10.37} & \underline{35.58} & 68.18 & 72.48 & \underline{49.45} & \textbf{28.80} & \textbf{58.12} & \underline{67.17} & \textbf{54.25} \\
            & ALPS      & \underline{8.103}& \textbf{10.29} & 33.28 & 65.19 & 68.75 & 45.96 & 26.20 & 55.96 & 65.98 & 51.62 \\
            & SAFE      & \textbf{8.043} & 10.47 & 31.57 & 66.84 & 68.04 & 48.55 & 23.40 & 53.07 & 65.04 & 50.93 \\
            & SparseLLM & 8.679 & 11.04 & 34.90 & \underline{68.35} & \textbf{75.14} & 48.28 & 26.20 & \underline{56.68} & 66.46 & \underline{53.71} \\
            \rowcolor{gray!10}\cellcolor{white}
            & \ours     & 9.20 & 11.47 & 32.25&	64.69&	69.42&	49.90&	27.40&	53.07&	63.22&	51.42 \\
        \bottomrule
    \end{tabular}}
    % \vspace{-3em}
\end{table}

We first evaluate \ours for its adaptivity to N:M semi-structured sparsity, a setting designed for some current hardwares to accelerate computations \citep{sunsimple,fang2024maskllm}.
The results of both perplexity and zero-shot prediction accuracy are reported in \cref{tab:llama7b-semi}.
\ours is roughly on par with existing methods, and yet, it is noteworthy that these 2:4 and 4:8 sparsity patterns only ensure 50\% sparsity.
More importantly, these results indicate that \ours can easily adapt to arbitrary constraints of moderate sparsity levels without much trouble.

We also compare \ours with non-uniform sparsity allocation based pruning methods.
Specifically, we compare to OWL \citep{yin2024outlier} that allocates sparsity based on outlier distributions and to EvoPress \citep{sieberling2024evopress} that uses an evolutionary search strategy to determine the non-uniform sparsity levels over different layers.
We further set up a method that overrides \ours with the mask found by the evolutionary strategy of EvoPress.
Note that the sparsity level is set to be 70\%;
it is simply because these methods only works or reports up to this level.
The results are presented in \cref{tab:llama8b_evopress70}.
% There are two different motivations behind these experiments: (i) to make a direct comparison to non-uniform pruning methods and (ii) whether or not a claim-to-be optimal mask found by evolutionary search ca
% which assign heterogeneous per-layer budgets to improve high-sparsity performance.  
% OWL \citep{yin2024outlier} determines budgets by analyzing outlier distributions, while EvoPress \citep{sieberling2024evopress} uses evolutionary search on top of a layer-wise solver such as SparseGPT.
One can see that \ours substantially outperforms OWL and shows an improvement over EvoPress as well:
to elaborate, for instance, it achieves the C4 perplexity of $29.09$, compared to $33.72$ for EvoPress and $52.32$ for OWL.
Notably, adopting the non-uniform mask found by EvoPress within \ours yields some gains over the EvoPress itself, but it still falls short of the uniform allocation in \ours, demonstrating the strength of our surrogate-free global formulation.
% \vspace{-1em}
% \begin{figure}[t]
%     \centering
%     % --- Left: Figure (Caption width matches image) ---
%     \begin{minipage}[c]{0.45\linewidth}
%         \centering
%         \label{fig:ppl_90_single_row}
%         \captionsetup{type=figure, width=0.6\linewidth} 
%         \caption{Perplexity vs. data count on OPT-125M at 90\% sparsity.}
%         % Set the image width
%         \includegraphics[width=0.6\linewidth]{fig/figure_7.pdf} 
%     \end{minipage}
%     \begin{minipage}[c]{0.45\linewidth}
%         \centering
%         \resizebox{0.6\linewidth}{!}{%
%             \scriptsize
%             \begin{tabular}{c cc}
%                 \toprule
%                 \textbf{Sparsity} & $\circ$ & $\times$ \\
%                 \midrule
%                 70\% & \textbf{28.24} & 29.44 \\
%                 80\% & \textbf{37.50} & 40.06 \\
%                 90\% & \textbf{48.69} & 65.41 \\
%                 \bottomrule
%             \end{tabular}%
%         }
%         % Match the caption width to the table width
%         \captionsetup{type=table, width=0.6\linewidth} 
%         \captionof{table}{Perplexity of different projection methods on LLaMA-3.2-3B.}
%         \label{tab:ablation_momentum}
%     \end{minipage}
% \end{figure}

\newpage 

\subsection{Ablations}
\label{subsec:ablation}

% \begin{wrapfigure}{r}{0.35\textwidth}{%
%     \vspace{-2em}
%     \centering
%     \begin{minipage}{\linewidth} 
%         \centering
%         \includegraphics[width=\linewidth]{fig/figure_7.pdf}
%         \caption{Perplexity vs. Data count measured on OPT-125M at 90\% sparsity.}
%         \label{fig:ppl_90_single_row}
%     \end{minipage}
    
%     \vspace{1em} % 그림과 표 사이의 수직 간격
%     \begin{minipage}{\linewidth} 
%         \centering
%         \captionof{table}{Perplexity of different projection methods on LLaMA-3.2-3B.}
%         \label{tab:ablation_momentum} 
%         \begin{tabular}{c cc}
%             \toprule
%             \textbf{Sparsity} & $\circ$ & $\times$ \\
%             \midrule
%             70\% & \textbf{28.24} & 29.44 \\
%             80\% & \textbf{37.50} & 40.06 \\
%             90\% & \textbf{48.69} & 65.41 \\
%             \bottomrule
%         \end{tabular}
%     \end{minipage}
%     }
% \end{wrapfigure}

In this section, we present two ablation analyses on (i) the choice of objective comparing the next token prediction (NTP) against the reconstruction error minimization (REM), and (ii) the projection step contrasting our objective-aware variant with the standard projection method.

\begin{wrapfigure}{r}{0.3\textwidth}
    \vspace{-1em}
    \centering
    \includegraphics[width=0.8\linewidth]{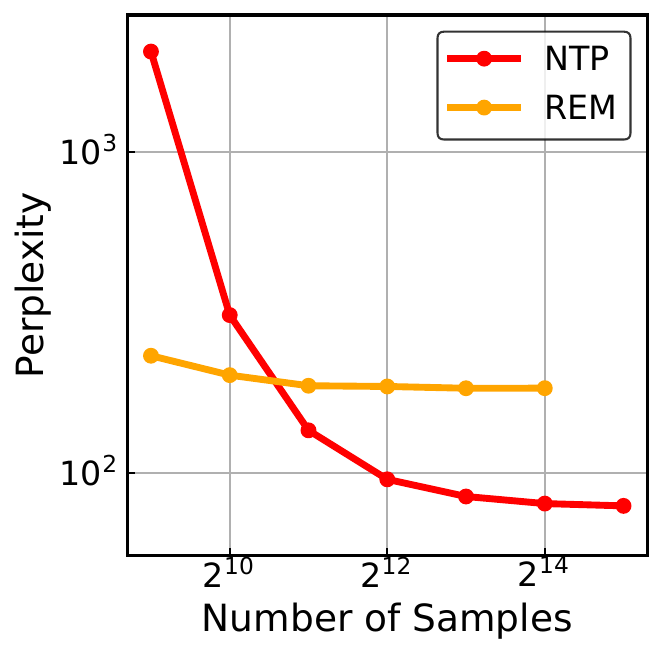}
    \caption{
        Effect of NTP on data efficiency and perplexity.
    }
    \label{fig:ppl_90_single_row}
\end{wrapfigure}

Specifically, we first set up a experiment where we measure how effectively our surrogate-free approach with NTP make use of data to preserve the original model performance while vayring the number of data samples.
We compare that to the existing REM approach.
The results are plotted in \cref{fig:ppl_90_single_row}.
While REM tend to perform better than NTP at low data regime, but it soon starts to saturate as data counts increases producing diminishing returns.
This is in stark constrat to NTP by which pruning performance keeps on improving quite drastically with more data.
Notably, REM requires memory to store dense model predictions, which can grow prohibitively large as with large data.
By contrast, NTP naturally benefits from additional data and continues to improve, enabling scalable LLM sparsity.
This in part reveals the inherent limitation of surrogate objectives.

\begin{wraptable}{r}{0.3\textwidth}
    \vspace{-1em}
    \centering
    \scriptsize
    \caption{Effectiveness of geometric projection (\checkmark).}
    \label{tab:ablation_momentum}
    \resizebox{\linewidth}{!}{%
        \begin{tabular}{c cc}
            \toprule
            \textbf{Sparsity} & \xmark & \checkmark \\
            \midrule
            70\% & 29.44 & \textbf{28.24} \\
            80\% & 40.06 & \textbf{37.50} \\
            90\% & 65.41 & \textbf{48.69} \\
            \bottomrule
        \end{tabular}
        % \vspace{-5em}
    }
\end{wraptable}

We also evaluate the effectiveness of the objective-aware projection on high-sparsity regimes.
Specifically, we measure the perplexity of LLaMA-3.2-3B model pruned for 70-90\% sparsity levels by turning on and off of the projection and report the results in \cref{tab:ablation_momentum}.
The benefit of objective-aware projection grows with sparsity: perplexity gap increases from 1.20 at 70\% sparsity to 2.56 at 80\%, and widens further at 90\%.  
This demonstrates that incorporating objective-aware importance into the projection step can be beneficial particularly in high sparsity regimes. 

\subsection{Numerical results}
Here we provide numerical results used to produce visual figures on the main text, followed by additional results reporting LLaMA-2-7B/13B zero-shot task accuracy.
\begin{table}[t]
    \centering
    \caption{Perplexity ($\downarrow$) of various models pruned with different methods across sparsity levels. Dense performance is shown under each model name (\texttt{Wiki} / \texttt{C4}). 
    Results for SparseLLM on Gemma-2-2B and LLaMA-2-13B are omitted due to implementation limitations (e.g., architectural incompatibility, out-of-memory errors). We could not obtain results of SparseLLM in Gemma-2-2b, Llama-3.2-3B, Llama-2-13B.} 
    \label{tab:main-table}
    \resizebox{\linewidth}{!}{%
    \begin{tabular}{cl *{10}{c}}
        \toprule
            &
            & \multicolumn{2}{c}{50\%} & \multicolumn{2}{c}{60\%}
            & \multicolumn{2}{c}{70\%} & \multicolumn{2}{c}{80\%} & \multicolumn{2}{c}{90\%} \\
        \cmidrule(lr){3-4}\cmidrule(lr){5-6}\cmidrule(lr){7-8}\cmidrule(lr){9-10}\cmidrule(lr){11-12}
         \textbf{Model} & \textbf{Method} & \texttt{Wiki} & \texttt{C4} & \texttt{Wiki} & \texttt{C4}
           & \texttt{Wiki} & \texttt{C4} & \texttt{Wiki} & \texttt{C4} & \texttt{Wiki} & \texttt{C4} \\
        \midrule
                % ===== OPT-125M =====
        \multirow{8}{*}{\shortstack{OPT-125M \\ \scriptsize (Dense: 27.65 / 26.56)}}
            & Magnitude & 193.4 & 141.0 & 920.0 & 598.2 & 3806 & 2263 & 4890 & 3213 & 6613 & 4475 \\
            & Wanda     & 38.93 & 34.91 & 77.85 & 63.33 & 351.8 & 248.9 & 1912 & 1066 & 4940 & 3126 \\
            & SparseGPT & 37.02 & 33.51 & 60.90 & 49.83 & 239.2 & 156.3 & 2072 & 1050 & 6131 & 2443 \\
            & L-ADMM    & \underline{33.02} & 31.21 & 45.04 & 38.49 & 100.5 & 74.61 & 580.8 & 315.8 & 3427 & 1350  \\
            & ALPS      & \textbf{32.70} & 30.91 & \underline{43.07} & \underline{36.94} & \underline{90.85} & \underline{66.28} & \underline{484.8} & \underline{267.7} & \underline{2524} & \underline{1094} \\
            & SAFE      & 33.88 & \underline{30.54} & 47.21 & 37.46 & 120.1 & 75.2 & 1254  & 726.8 & 5382 & 2331 \\
            & SparseLLM & 37.11 & 33.19 & 57.47 & 46.64 & 199.2 & 131.7 & 1576 & 752.2 & 4730 & 1825 \\
            \rowcolor{gray!10}\cellcolor{white}
            & \ours     & 37.68 & \textbf{30.10} & \textbf{41.5} & \textbf{33.22} & \textbf{49.57} & \textbf{39.86} & \textbf{65.30} & \textbf{47.74} & \textbf{95.33} & \textbf{62.28} \\
        \midrule
        \multirow{8}{*}{\shortstack{OPT-1.3B \\ \scriptsize (Dense: 14.62 / 16.07)}}
            & Magnitude & 1712 & 403.3 & 9392 & 5066 & 9442 & 6498 & 1.6e4 & 1.1e4 & 2.9e4 & 1.8e4 \\
            & Wanda     & 18.42 & 20.62 & 26.82 & 28.77 & 105.7 & 94.98 & 2504 & 1181 & 1.3e4 & 8447 \\
            & SparseGPT & 17.45 & 19.25 & 24.02 & 23.30 & 50.52 & 46.11 & 947.9 & 406.7 & 6472 & 2843 \\
            & L-ADMM      & 26.62 & 26.26 & 32.35 & 30.28 & 61.10 & 49.52 & 595.9 & 289.5 & 5659 & 2298 \\
            & ALPS      & \underline{16.78} & 18.59 & \underline{20.58} & 21.52 & 35.77 & 34.09 & \underline{285.7} & \underline{158.4} & \underline{4590} & \underline{1844} \\
            & SAFE      & \textbf{16.38} & \textbf{17.75} & \textbf{19.63} & \textbf{19.93} & \underline{31.17} & \underline{27.52} & 387.1 & 222.3 & 1.3e4 & 7544 \\
            & SparseLLM & 17.73 & 19.40 & 23.23 & 24.03 & 56.36 & 47.96 & 861.7 & 372.0 & 5535 & 2217 \\
            \rowcolor{gray!10}\cellcolor{white}
            & \ours     & 18.99 & \underline{18.45} & 22.11 & \underline{20.2} & \textbf{27.13} & \textbf{24.43} & \textbf{36.89} & \textbf{31.51} & \textbf{61.52} & \textbf{45.39}  \\
        \midrule
        \multirow{8}{*}{\shortstack{Gemma-2-2B \\ \scriptsize (Dense: 8.71/ 13.16 )}}
            & Magnitude & 51.66 & 57.68 & 2178 & 2064 & 4.4e7 & 3.5e6 & 2.5e9 & 2.4e8 & 5.0e9 & 2.3e9 \\
            & Wanda     & 12.07 & 17.49 & 21.39 & 32.40 & 117.5 & 152.0 & 994.6 & 855.6 & 1.1e4 & 5524  \\
            & SparseGPT & 11.58 & 16.67 & 16.53 & 23.44 & 34.73 & 47.43 & 147.7 & 160.5 & 983.1 & 776.5 \\
            & L-ADMM    & \underline{11.02} & \underline{15.84} & \underline{14.65} & 20.84 & 26.91 & 38.32 & 86.64 & 110.8 & 308.2 & 300.3 \\
            & ALPS      & \textbf{10.93} & \textbf{15.77} & \textbf{14.42} & \underline{20.32} & \underline{24.96} & 35.08 & 73.50 & 94.26 & \underline{238.5} & \underline{254.3} \\
            & SAFE      & 11.61 & 16.21 & 15.22 & 20.32 & 25.67 &\underline{33.39} & \underline{68.55} & \underline{75.22} & 432.7 & 345.0 \\
            & SparseLLM & --- & ---  & ---  & --- & ---  & ---  & ----  & ---  & ---  & ---  \\
            \rowcolor{gray!10}\cellcolor{white}
            & \ours     & 13.05	& 17.22 & 15.93 & \textbf{19.83} & \textbf{21.22} & \textbf{24.55} & \textbf{30.29} & \textbf{31.68} & \textbf{49.37} & \textbf{44.93} \\
        \midrule
        % ===== LLaMA-3B =====
        \multirow{8}{*}{\shortstack{LLaMA-3.2-3B \\ \scriptsize (Dense: 7.81 / 11.32)}}
            & Magnitude & 139.4 & 216.1 & 1.5e4 & 1.4e4 & 1.0e5 & 8.1e5 & 3.5e5 & 3.5e5 & 3.0e5 & 2.4e5 \\
            & Wanda     & 13.01 & 19.08 & 31.39 & 42.53 & 142.4 & 168.1 & 3859 & 1821 & 1.4e4 & 8766 \\
            & SparseGPT & 12.27 & 17.41 & 23.38 & 30.47 & 86.88 & 84.12 & 292.9 & 237.1 & 1807 & 1094 \\
            & L-ADMM      & 11.56 & 16.32 & 19.06 & 24.84 & 45.48 & 53.30 & 160.4 & 126.9 & 760.5 & 509.5 \\
            & ALPS      & \underline{11.31} & \underline{15.88} & 18.16 & 22.83 & \underline{41.79} & \underline{46.48} & \underline{166.32} & \underline{109.0} & \underline{542.0} & \underline{367.0} \\
            & SAFE      & \textbf{10.68} & \textbf{15.51} & \underline{16.76} & \underline{22.57} & 50.78 & 57.86 & 330.9 & 267.2 & 3410 & 2343 \\
            & SparseLLM & --- & ---  & ---  & ---  & ---  & ---  & ---  & ---  & ---  & --- \\
            \rowcolor{gray!10}\cellcolor{white}
            & \ours     & 12.06 & 17.34 & \textbf{16.65} & \textbf{21.73} & \textbf{24.07} & \textbf{28.24} & \textbf{36.25} & \textbf{37.50} & \textbf{50.88} & \textbf{48.69} \\
                \midrule
        % ===== LLaMA-7B =====
        \multirow{8}{*}{\shortstack{LLaMA-2-7B \\ \scriptsize (Dense: 5.47 / 7.26)}}
            & Magnitude & 16.03 & 21.34 & 1924 & 2063 & 5.0e4 & 2.8e4 & NaN & NaN & NaN & NaN \\
            & Wanda     & 6.92 & 9.24 & 10.79 & 13.99 & 76.32 & 81.08 & 4096 & 2673 & 2.0e4 & 1.0e4 \\
            & SparseGPT & 7.01 & 9.23 & 10.20 & 12.93 & 27.12 & 30.94 & 107.3 & 100.8 & 1430 & 864.5 \\
            & L-ADMM    & 6.80 & 8.97 & 9.40 & 11.47 & 20.56 & 22.20 & 60.78 & 58.63 & 400.5 & 287.1  \\
            & ALPS      & \underline{6.86} & \underline{9.02} & 9.33 & \textbf{11.30} & \underline{19.39} & \underline{20.37} & \underline{48.43} & \underline{47.22} & \underline{248.8} & \underline{180.9} \\
            & SAFE      & \textbf{6.72} & \textbf{8.87} & \textbf{9.02} & 11.40 & 86.80 & 48.54 & 8.1e5 & 5.3e5 & 1.6e4 & 1.6e4  \\
            & SparseLLM & 7.23 & 9.51 & 10.74 & 13.25 & 37.65 & 35.00 & 126.5 & 94.28 & 1267 & 648.0 \\
            \rowcolor{gray!10}\cellcolor{white}
            & \ours     & 7.5 & 9.81 & \underline{9.16} & \underline{11.34} & \textbf{13.20} & \textbf{14.08} & \textbf{20.83} & \textbf{19.56} & \textbf{26.97} & \textbf{23.14} \\
        \midrule
        % ===== LLaMA-13B =====
        \multirow{8}{*}{\shortstack{LLaMA-2-13B \\ \scriptsize (Dense: 4.88 / 6.73)}}
            & Magnitude & 6.83 & 9.38 & 11.82 & 14.62 & 214.2 & 191.9 & 3.9e4 & 4.9e4 & 7.5e4 & 6.5e4 \\
            & Wanda     & 5.97 & 8.30 & 8.40 & 11.53 & 45.37 & 56.27 & 1004 & 838.8 & 2.2e4 & 1.3e4 \\
            & SparseGPT & 6.03 & 8.22 & 8.27 & 10.93 & 19.79 & 23.47 & 97.82 & 79.17 & 1442 & 984.1 \\
            & L-ADMM      & 5.92 & 8.11 & 7.57 & 10.05 & 14.81 & 17.56 & 44.78 & 44.42 & 391.1 & 242.1 \\
            & ALPS      & \underline{5.90} & \underline{7.99} & \underline{7.56} & \underline{9.92} & 14.17 & 16.28 & \underline{38.44} & \underline{36.78} & \underline{231.3} & \underline{152.1} \\
            & SAFE      & \textbf{5.73} & \textbf{7.82} & \textbf{6.90} & \textbf{9.24} & \underline{12.47} & \underline{14.57} & 93.49 & 73.25 & 2122 & 1388 \\
            & SparseLLM & ---  & ---  & ---  & ---  & ---  & ---  & ---  & ---  & ---  & ---  \\
            \rowcolor{gray!10}\cellcolor{white}
            & \ours     & 6.54 & 8.78 & 7.96 & 9.93 & \textbf{11.14} & \textbf{12.20} & \textbf{17.21} & \textbf{16.60} & \textbf{30.19} & \textbf{25.07}\\
        \bottomrule
    \end{tabular}}
\end{table}

\begin{table*}[t]
    \centering
    \scriptsize
    \setlength{\tabcolsep}{4pt}
    \caption{Zero-shot accuracy (\%) of Llama-2-7B across multiple tasks, in various sparsity regime (50\%-90\%).}
    \label{tab:llama7b-zeroshot}
    \resizebox{\linewidth}{!}{%
    \begin{tabular}{cl*{8}{>{\centering\arraybackslash}p{0.9cm}}}
        \toprule
        & & \multicolumn{8}{c}{\textbf{Tasks}} \\
        \cmidrule(lr){3-9}
        \textbf{Sparsity}& \textbf{Method} & ARC-C & ARC-E & BoolQ & HellaSwag & OBQA & RTE & Winogrande & Avg \\
        \midrule
        0\% & Dense &  43.35&	76.26&	77.68&	57.14&	31.40&	62.82&	69.06 & 59.67  \\
        \midrule
        \multirow{8}{*}{50\%}
            & Magnitude & 34.90 & 64.02 & 62.91 & 49.13 & 26.80 & 57.04 & 63.22 & 51.14 \\
            & Wanda     & 39.25 & 72.22 & 75.17 & 52.64 & 30.60 & 53.43 & 67.17 & 55.78 \\
            & SparseGPT & 38.23 & 71.34 & \underline{75.99} & 52.70 & 29.80 & 56.32 & \textbf{69.77} & 56.31 \\
            & L-ADMM    & \underline{39.68} & \underline{72.77} & \textbf{76.24} & \underline{53.35}& \textbf{31.40} & \textbf{61.37} & \underline{69.30} & \textbf{57.73} \\
            & ALPS      & \textbf{40.61} & \textbf{72.90} & 75.44 & \textbf{53.37} & \underline{30.80} & \underline{57.76} & 68.98 & \underline{57.14} \\
            & SAFE      & 38.14 & 72.14 & 74.83 & 52.15 & 26.00 & 57.04 & 66.77 & 55.30 \\
            & SparseLLM & 38.05 & 71.25 & 75.14 & 52.66 & 29.60 & 53.43 & 69.30 & 55.63 \\
            \rowcolor{gray!10}\cellcolor{white}
            & \ours & 39.42&	71.30&	73.03&	53.12&	29.40&	58.48&	66.54&	56.39 \\
        \midrule
        \multirow{8}{*}{60\%}
            & Magnitude & 25.17 & 44.87 & 47.80 & 35.00 & 20.00 & 50.90 & 53.12 & 39.55 \\
            & Wanda & 30.63 & 64.44 & 65.51 & 43.51 & 25.80 & \underline{54.15} & 64.01 & 49.72 \\
            & SparseGPT & 31.57 & 64.06 &\textbf{72.57} & 45.0 & 25.80 & 53.43 & 65.51 & 51.13 \\
            & L-ADMM & \underline{34.13} & \textbf{66.50} & 70.43 & 47.29 & 26.60 & \textbf{55.60} & \textbf{66.61} & \textbf{52.45} \\
            & ALPS & \textbf{34.38} & \underline{66.33} & 70.64 & \underline{47.81} & \textbf{27.2} & \underline{54.15} & \underline{66.29} & \underline{52.40} \\ 
            & SAFE      & 31.14 & 64.14 & \underline{71.10} & 46.43 & 24.00 & \underline{54.15} & 62.98 & 50.57 \\
            & SparseLLM & 32.59 & 64.52 & 70.86 & 45.24 & 25.80 & 53.79 & 66.14 & 51.28 \\
            \rowcolor{gray!10}\cellcolor{white}
            & \ours & 32.08&	65.83&	67.52&	\textbf{49.19}&	\underline{26.00}&	53.80&	63.38&	51.41 \\
        \midrule
            \multirow{8}{*}{70\%} 
                & Magnitude & 22.87 & 27.82 & 37.95 & 25.90 & 17.20 & 53.07 & 49.25 & 33.43 \\
                & Wanda & 18.6 & 30.01 & 57.28 & 28.04 & 12.0 & 52.71 & 48.86 & 35.36 \\
                & SparseGPT & 22.01 & 42.34 & \textbf{65.14}& 33.04 & 16.8 & 52.71 & 57.7 & 41.39 \\
                & L-ADMM & 23.81 & 50.63 & 63.21 & 36.57 & 20.40 & \underline{54.15} & 60.77 & 44.22 \\
                & ALPS     & \underline{25.51} & \underline{52.78} & \underline{63.46} & \underline{37.54} & \underline{20.8} & 53.43 &\underline{61.72}& \underline{45.03}  \\ 
                & SAFE & 24.23 & 45.62 & 43.76 & 34.74 & 18.40 & \underline{52.71} & 53.12 & 38.94  \\
                & SparseLLM & 20.90 & 40.32 & 61.87 & 32.74 & 16.0 & \textbf{54.51} & 57.46 & 40.54 \\
                \rowcolor{gray!10}\cellcolor{white}
            & \ours & \textbf{27.13} & \textbf{55.81} &	63.61&	\textbf{43.16}&	\textbf{22.40}&	52.71&	58.64&	\textbf{46.21}\\
            \midrule
            \multirow{8}{*}{80\%}
            & Magnitude & \textbf{22.35} & 25.38 & 43.67 & 25.72 & 13.00 & 46.57 & 51.62 & 32.62  \\
            & Wanda & 20.82 & 26.98 & 37.83 & 25.89 & \underline{15.0} & 52.71 & 49.25 & 32.64 \\
            & SparseGPT & 17.92 & 27.95 & 38.07 & 27.51 & 12.0 & \textbf{53.07} & 49.01 & 32.22 \\
            & L-ADMM & 18.26 & 29.29 & 57.49 & 28.33 &	13.00 & \textbf{53.07} & \underline{51.22} & 35.81 \\
            & ALPS & 19.37 & \underline{32.07} & \textbf{61.1} & \underline{29.06} & 12.6 & 52.71 & 50.91  & \underline{36.83}  \\
            & SAFE & \underline{21.76} & 25.80 & 37.83 & 26.01 & 14.00 & 52.71 & 49.80 & 32.56 \\
            & SparseLLM & 18.09 & 28.70 & 43.55 & 27.57 & 11.6 & 52.71 & 48.86 & 33.01 \\
            \rowcolor{gray!10}\cellcolor{white}
            & \ours & 20.99 & \textbf{44.61} &  \underline{60.67} &  \textbf{34.02} &  \textbf{16.80} & 52.71 & \textbf{53.20} &  \textbf{40.43} \\
            \midrule
            \multirow{8}{*}{90\%}
            & Magnitude & \textbf{22.78} & 25.93 & \underline{39.17} & 25.53 & \underline{16.0} & 47.29 & 50.12 & 32.40  \\
            & Wanda & \underline{21.67} & 25.46 & 37.83 & 25.83 & 15.2 & 47.29 & 49.33 & 31.8\\
            & SparseGPT & 20.65 & 26.77 & 37.83 & 25.7 & 13.0 & \underline{52.71} & \underline{50.59} & 32.46 \\
            & L-ADMM & 19.97 & 26.14 & 37.83 & 26.46 & 13.60 & 51.62 & 47.51 & 31.88 \\
            & ALPS & 19.45 & \underline{26.89} & 37.8 & \underline{26.81} & 12.8 & \textbf{53.79} & 46.65 & 32.03 \\
            & SAFE & 21.84 & 26.52 & 37.83 & 25.91 & 15.80 & \underline{52.71} & 47.83 & \underline{32.63} \\
            & SparseLLM & 20.56 & 25.72 & 37.83 & 25.94 & 13.8 & \underline{52.71} & 46.96 & 31.93 \\
            \rowcolor{gray!10}\cellcolor{white}
            & \ours &18.52 & \textbf{41.33} & \textbf{57.25} &  \textbf{31.54} & \textbf{16.60} &  \underline{52.71} &  \textbf{51.70} & \textbf{38.52} \\
        \bottomrule
    \end{tabular}%
    }
\end{table*}

\begin{table*}[t] %
    \centering
    \small
    \setlength{\tabcolsep}{6pt}
    \caption{Zero-shot accuracy (\%) of Llama-2 13B across multiple tasks, under various sparsity levels. We could not obtain SparseLLM in Llama-2-13B.}
    \label{tab:llama13b-zeroshot}
    \resizebox{\linewidth}{!}{%
    \begin{tabular}{cl*{8}{>{\centering\arraybackslash}p{1.1cm}}}
        \toprule
        \textbf{Sparsity} & \textbf{Method} & \multicolumn{8}{c}{\textbf{Tasks}} \\
        \cmidrule(lr){3-9}
        & & ARC-C & ARC-E & BoolQ & HellaSwag & OBQA & RTE & Winogrande & Avg \\
        \midrule
        0\% & Dense & 48.46 & 79.38 & 80.55 & 60.04 & 35.20 & 65.34 & 72.14 & 63.02 \\
        \midrule
        \multirow{7}{*}{50\%}
            & Magnitude &  38.48 & 70.58 & 57.65 & 54.39 & 27.80 & 55.96 & 65.35 & 52.89 \\
            & Wanda     &43.09 &\textbf{76.30} & 80.95 & \underline{56.96} & 31.20 & 60.65 & 71.43 & 60.08 \\
            & SparseGPT & 42.41 & 74.96 & \textbf{81.53} & 55.95 & 31.00 & \textbf{64.26} & 71.35 & 60.21 \\
            & L-ADMM    & \underline{43.17} & 75.84 & \textbf{82.29} & 56.51 & 32.00 & \underline{63.18} & \underline{71.98} & \textbf{60.71} \\
            & ALPS      & 42.66 & \textbf{76.30} & 81.22 & 56.71 & \underline{32.60} & 62.82 & \textbf{72.14} & \underline{60.64} \\
            & SAFE      & 41.64 & 75.84 & 80.40 & 56.59 & 30.60 & 60.65 & 69.14 & 59.27 \\
            \rowcolor{gray!10}\cellcolor{white}
            & \ours & \textbf{43.68} &	75.12&	77.71&	\textbf{57.45}&	\textbf{33.80}&	55.60&	70.00&	59.05 \\
        \midrule
        \multirow{7}{*}{60\%}
        & Magnitude 
        & 27.13 & 56.14 & 47.49 & 44.66 & 21.80 & 52.71 & 57.46 & 43.93 \\
        & Wanda 
        & 37.97 & 68.81 & 77.16 & 48.71 & 28.20 & 59.57 & 68.19 & 55.51 \\
        & SparseGPT 
        & 36.01 & 69.40 & 78.72 & 49.38 & 27.4 & 57.76 & \underline{70.56} & 55.60 \\
        & L-ADMM & 39.76 & \textbf{72.98} & \underline{80.70} & 51.49 & 29.8 & \underline{59.93} & 70.32 &	\underline{57.86} \\
        & ALPS     
        & \underline{40.44} & \underline{72.93} & \textbf{81.68} & 51.97 & \underline{30.80} & \textbf{60.29} & \textbf{71.90} & \textbf{58.58} \\
        & SAFE  
        & 36.95 & 72.43 & 78.38 & \underline{52.09} & 28.80 & 57.40 & 67.88 & 56.28 \\

        \rowcolor{gray!10}\cellcolor{white}
        & \ours 
        &  \textbf{40.53} &	70.67&	75.17&	\textbf{54.22} &	\textbf{31.00}&	53.43&	67.80&	56.12 \\
        \midrule
        \multirow{7}{*}{70\%}
        & Magnitude 
        & 20.65 & 31.31 & 38.65 & 27.53 & 14.60 & 52.71 & 49.25 & 33.53 \\
        & Wanda 
        & 18.43 & 36.45 & 62.35 & 29.25 & 13.0 & 52.71 & 50.83 & 37.57 \\
        & SparseGPT 
        & 25.34 & 49.58 & 67.86 & 36.27 & 20.2 & 52.71 & 60.93 & 44.70 \\
        & L-ADMM & 27.56 & 59.64 & 69.76 & 40.05 & 24.00 & \textbf{53.43} & \textbf{65.35} & 48.54 \\
        & ALPS     
        & 29.61 & 61.20 & 70.09 & 40.86 & \textbf{26.6} & \underline{53.07} & \underline{64.56} & \underline{49.43} \\
        & SAFE
        & 29.78 & 61.07 & 69.17 & 41.62 & 20.20 & 52.71 & 58.96 & 47.64 \\

        \rowcolor{gray!10}\cellcolor{white}
        & \ours &  \textbf{34.13} &	\textbf{62.42} & \textbf{70.12} &	\textbf{47.52} &	24.80&	52.71&	60.38&	\textbf{50.30} \\
        \midrule
        \multirow{7}{*}{80\%}
        & Magnitude 
        & \underline{21.84} & 25.63 & 41.80 & 25.88 & \underline{14.80} & \textbf{53.07} & 49.25 & 33.18 \\
        & Wanda 
        & 20.48	& 26.26 & 37.83 & 26.81 & 12.6 & \underline{52.71} & 50.04 & 32.39 \\
        & SparseGPT 
        & 19.62 & 28.79 & 59.05 & 27.77 & 12.8 & \underline{52.71} & 49.33 & 35.72 \\
        & L-ADMM & 19.11 & 33.80 & 62.14 & 29.67 & 14.60 & \underline{52.71} & 53.28 & 37.90 \\
        & ALPS     
        & 20.05 & \underline{35.99} & \underline{62.17} & 30.65 & 14.0 & \underline{52.71} & \textbf{54.93} & \underline{38.64} \\
        & SAFE 
        & 18.34 & 28.37 & 40.64 & 27.44 & 12.80 & \underline{52.71} & 50.51 & 32.97 \\
        \rowcolor{gray!10}\cellcolor{white}
        & \ours &   \textbf{24.32}&	\textbf{50.97} &	\textbf{63.52} & \textbf{38.03} &	\textbf{19.60} & \underline{52.71} & \underline{53.75} & \textbf{43.27} \\
        \midrule
        \multirow{7}{*}{90\%}
        & Magnitude 
        & \underline{21.42} & 24.87 & 44.16 & 25.72 & \underline{15.0} & 46.57 & \underline{51.78} & 32.79 \\
        & Wanda 
        & 21.33 & 25.93 & 37.83 & 25.80 & 13.8 & \underline{52.71} & 51.54 & 32.70 \\
        & SparseGPT 
        & 21.08 & 25.76 & \textbf{58.62} & 25.87 & 13.8 & 52.35 & 49.49 & \underline{35.28} \\
        & L-ADMM & 19.88 & 26.01 & 39.45 & 27.08 & 13.80 & \textbf{53.79} & 50.04 & 32.86 \\
        & ALPS     
        & 18.94 & \underline{26.94} & 43.52 & \underline{27.37} & 13.4 & \underline{52.71} & 48.30 & 33.02 \\
        & SAFE 
        & \textbf{22.10} & 25.76 & 37.83 & 26.02 & 14.20 &  \underline{52.71} & \textbf{53.43} & 33.15 \\
        \rowcolor{gray!10}\cellcolor{white}
        & \ours &  19.03 & \textbf{36.15} & \underline{58.44} & \textbf{28.65} & \textbf{16.20} & \underline{52.71} & 50.43  & \textbf{37.37}\\
        \bottomrule
    \end{tabular}
    }
\end{table*}

\end{document}